\documentclass[accepted]{uai2022} 

\usepackage[table]{xcolor}
\usepackage{bm}
\usepackage{amsmath,amsfonts,amsthm, amssymb}
\usepackage{mathrsfs}  
\usepackage{array}
\usepackage{mathptmx}
\usepackage{xcolor, soul}
\usepackage{graphicx}
\usepackage{multirow}
 
\usepackage{bbm}
\usepackage{booktabs}
\usepackage{siunitx}
\usepackage{color}        

\usepackage{multirow}
\newcommand{\Qset}{\mathcal{Q}}
\newtheorem{example}{Example} 
\newtheorem{theorem}{Theorem}  
\newtheorem{proposition}[theorem]{Proposition}

\newtheorem{definition}[theorem]{Definition}

\usepackage{enumitem}
\setitemize{noitemsep,topsep=0pt,parsep=0pt,partopsep=0pt}
\usepackage[belowskip=-4pt,aboveskip=0pt]{caption}

\usepackage{hhline}
\usepackage{float} 

\usepackage[british]{babel}

\usepackage{mathtools} 
\usepackage{booktabs} 
\usepackage{tikz} 

\usepackage[round]{natbib}

\newcommand{\bx}{{\bf x}}

\title{Learning Choice Functions with Gaussian Processes }

\author[1]{Alessio Benavoli}
\author[2]{Dario Azzimonti}
\author[2]{Dario Piga}
\affil[1]{%
    School of Computer Science and Statistics\\
    Trinity College Dublin, Ireland
}
\affil[2]{%
    Dalle Molle Institute for Artificial Intelligence (IDSIA)\\
    USI/SUPSI\\
    Lugano, Switzerland
\vspace{0cm}}

\begin{document}
\maketitle

\begin{abstract}
In consumer theory, ranking available objects  by means of preference relations yields the most common description of individual choices. However, preference-based models assume that individuals: (1) 
give their preferences only between pairs of objects; (2) are always able to pick  the best preferred object. In many situations, they may be instead choosing out of a set with more than two elements and, 
because of lack of information and/or incomparability (objects with  contradictory characteristics), they may not able to select a single most preferred object.  To address these situations, we need a choice-model which allows an individual to express a set-valued choice. Choice functions provide such a mathematical framework. We propose a Gaussian Process model to learn choice functions from choice-data. The proposed model assumes a multiple utility representation of a choice function based on the concept of Pareto rationalization, and derives a strategy to learn both the number and the values of these latent multiple utilities. Simulation experiments demonstrate that the proposed model outperforms the state-of-the-art methods. 
\end{abstract}

\section{Introduction}
\label{sec:intro}
We are interested in learning the behavior of an individual (e.g., a consumer), we call her Alice, who is faced with the problem of choosing from among a set of objects, e.g., laptops:

\centerline{
\mbox{\huge $\big\{$} \includegraphics[width=0.9cm]{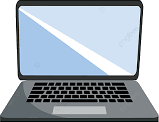},  \includegraphics[width=0.9cm]{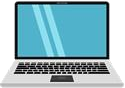},  \includegraphics[width=0.7cm]{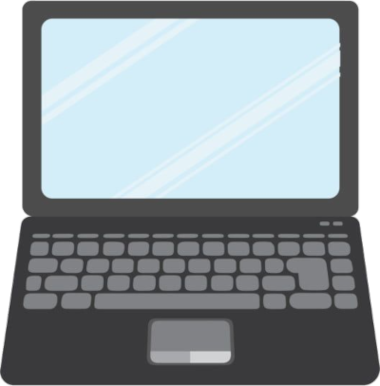},  \includegraphics[width=0.9cm]{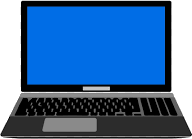} ,  \includegraphics[width=0.8cm]{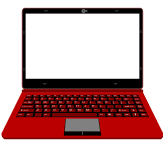}\mbox{\huge $\big\}$} }

This is an important problem for instance in computational advertising, and for personalisation of products and services. In consumer theory \citep{kreps1990course}, ranking available objects  by means of preference relations yields the most common description of individual choices:

\centerline{\includegraphics[width=0.9cm]{lap21.png} \mbox{\huge $ \stackrel{\succ}{}$} 
\includegraphics[width=0.8cm]{lap61.png}}

Preference-based models assume that  Alice is always able to pick  the best preferred object from a set of objects. However, in many contexts, Alice is faced with the problem of dealing with several contradictory primitives. For example, if the objects are laptops, Alice has to consider criteria such as: speed, drive-capacity and weight.
In other contexts, Alice may not have sufficient knowledge to pick up the best preferred object. For instance, in the laptop case, she may not know if she will use the laptop for simulations or for everyday office work.

 To deal with incomparability of objects, we need a choice-model which allows Alice to express a set-valued choice. \textit{Choice functions} provide such a mathematical framework, as well as allowing individuals to choose out of a set with more than two elements. For any
given set of objects $A$, they return the corresponding set-valued choice $C(A)$:

\centerline{
\mbox{\huge {\Large $A=$}$\big\{$} \includegraphics[width=0.9cm]{lap11.png},  \includegraphics[width=0.9cm]{lap21.png},  \includegraphics[width=0.7cm]{lap51.png},  \includegraphics[width=0.9cm]{lap41.png} ,  \includegraphics[width=0.8cm]{lap61.png}\mbox{\huge $\big\}$}}
\centerline{
\mbox{\huge {\Large$C(A)=$}$\big\{$}  \includegraphics[width=0.9cm]{lap21.png},   \includegraphics[width=0.9cm]{lap41.png} ,  \includegraphics[width=0.8cm]{lap61.png}\mbox{\huge $\big\}$}}

In the general interpretation of choice functions,  the statement that an object \includegraphics[width=0.5cm]{lap11.png} in $ A$ is rejected (that is, \includegraphics[width=0.5cm]{lap11.png}$\notin C(A)$) means that there is at least one object in $A$ that Alice strictly prefers over \includegraphics[width=0.5cm]{lap11.png}.\footnote{Alice is not required to tell us which object(s) in $C(A)$ dominate \includegraphics[width=0.5cm]{lap11.png}.} Instead,  any two objects in $C(A)$   are deemed to be incomparable by Alice. 

We represent each object by the feature vector ${\bf x}\in \mathbb{R}^c$ of its characteristics (e.g., for laptops, speed, weight etc.) and  propose a Gaussian process (GP) model to learn choice functions from choice-data $\{(C(A_s),A_s),~s=1,\dots,m\}$.

Our main contributions are:

 We propose a  generalisation of the preference learning model  by \cite{ChuGhahramani_preference2005} to choice functions. This  generalisation assumes a multiple utility representation of a choice function based on the concept of Pareto rationalization \citep{moulin1985choice}: Alice picks the Pareto optimal objects based on the value of latent multiple utilities. 

 Learning choice functions  via a Pareto embedding was originally proposed by \cite{pfannschmidt2020learning}, but using a hinge-loss and a neural network based model. We will show that our GP-based approach results into a more accurate and robust model.  
 
 The output of a choice function is a choice set, which is invariant with respect to permutations of its elements. We will show that this determines a number of challenges and propose ways to address this issue.

 Finally, we propose a method to learn the number of latent multiple utilities via Pareto Smoothed Importance sampling Leave-One-Out (PSIS-LOO) cross-validation \cite{vehtari2017practical}. Exact cross-validation requires re-fitting the model with different training sets. Instead, PSIS-LOO can be computed efficiently using  samples from the posterior.

\section{Background}
 To begin, let $\mathcal{X}$ represent some set of objects. It is quite typical in applications  to think of $\mathcal{X}$ as a subset of $\mathbb{R}^c$, where $c$ is the number of features of the object. 
 The standard way to model Alice (the consumer) is with a preference relation. We present Alice with pairs of objects, ${\bf x}_i,{\bf x}_j \in \mathcal{X}$, and ask her whether  ${\bf x}_i$ is better than ${\bf x}_j$ (see for instance     \cite{kreps1990course,furnkranz2010preferencebook,domshlak2011preferences}). If Alice says that  ${\bf x}_i$ is better than ${\bf x}_j$, we write ${\bf x}_i \succ {\bf x}_j$ and we say that ${\bf x}_i$ is strictly preferred to ${\bf x}_j$. In this paper, we will assume
that no draws are allowed -- no two distinct objects are equal -- and therefore focus on strict preference relations.
     
 An important result in \textit{preference theory}
 establishes conditions under which a preference relation can be numerically represented. We refer to a value function that represents preferences as a \textit{utility function}.

\begin{definition}
For any preference relation $\succ$  on $\mathcal{X}$, the function $u :\mathcal{X} \rightarrow \mathbb{R}$
represents $\succ$ if
\begin{equation}
 \label{eq:iifu}
 {\bf x}_i \succ {\bf x}_j ~~\textit{ iff }~~ u ({\bf x}_i ) > u ({\bf x}_j).
\end{equation}
We say that $u$ is a utility function for $\succ$.
\end{definition}
The relation $\succ$  admits a utility function representation iff\footnote{This result holds under some topological assumptions on $\mathcal{X}$, which are met when $\mathcal{X}=\mathbb{R}^c$ \citep{debreu1954representation}.} it is \cite[Ch.\ 2]{kreps1990course}:
\begin{itemize}
\item \textit{Asymmetric}:  if ${\bf x}_i \succ {\bf x}_j$ then  $\neg({\bf x}_j \succ {\bf x}_i)$;
\item \textit{Negatively transitive}: if ${\bf x}_i \succ {\bf x}_j$ then for any other element ${\bf x}_k \in \mathcal{X}$ either ${\bf x}_i \succ {\bf x}_k$ or ${\bf x}_k \succ {\bf x}_j$ or both.
\end{itemize}  
%
A strict preference relation is said to be \textit{consistent} -- Alice is  \textit{rational} -- when it satisfies the above two properties. It is immediate to verify that any consistent strict preference is also \textit{transitive} and \textit{acyclic} \cite[Ch.\ 2]{kreps1990course}.\footnote{Acyclic: if, for any finite number $n$, ${\bf x}_1 \succ {\bf x}_2$, ${\bf x}_2 \succ {\bf x}_3$, $\dots$, ${\bf x}_{n-1} \succ {\bf x}_n$ then ${\bf x}_n \neq {\bf x}_1$.}

 Typical PL models can be divided in two categories:  (1) those  assuming that the preference relation is consistent and aiming to learn the underlying latent utility function, e.g.,  \citep{ChuGhahramani_preference2005,houlsby2011bayesian,benavoli2020preferential}; (2) those solving the problem as an augmented binary classification problem or constrained classification (SVM), e.g., \citep{cohen1997learning,herbrich1998learning, aiolli2004learning,har2002constraint,fiechter2000learning}.
   
 In the first case, the goal is to learn  $u : \mathcal{X} \rightarrow \mathbb{R}$ from  $m$ preferential observations:
 $$
 \mathcal{D}_m = \{ {\bf x}^{(s)}_l \succ {\bf x}^{(s)}_r:~~ s = 1,\dots,m\},
 $$ 
 with ${\bf x}^{(s)}_l\neq {\bf x}^{(s)}_r$, ${\bf x}^{(s)}_l,{\bf x}^{(s)}_r \in \mathcal{X}$ and the subscripts $l,r$ stay for ``left hand side'' term and, respectively, ``right hand side'' term of the inequality. These PL models assume a Gaussian Process (GP) prior  on the latent utility $u$.

 These PL models also account for the fact that Alice's preferences may fail to satisfy asymmetry and/or negative transitivity for a
number of reasons. For instance, \textit{Limit of discernibility:} Alice may make mistakes when comparing two objects ${\bf x}_i,{\bf x}_j$ whose difference in utility is small (e.g., errors are inversely proportional to $|u({\bf x}_i)-u({\bf x}_j)|$). \textit{Noise:} the observed utility function differs from
the true utility function due to disturbances
(e.g., $o({\bf x}_i)=u({\bf x}_i)+\text{noise}$).
These issues were originally studied by \cite{luce1956semiorders} and, respectively, \cite{thurstone2017law} for Gaussian noise.

To account for both these issues, we can assume that the probability of correctly stating ${\bf x}_i \succ {\bf x}_j$  is a function of the difference $u({\bf x}_i)-u({\bf x}_j)$. This probability can be  modelled by  the following likelihood:
\begin{equation}
    \label{eq:likelcdf0}
p({\bf x}_i \succ {\bf x}_j|u)=\Phi\left(\frac{u({\bf x}_i)-u({\bf x}_j)}{\sigma}\right),
\end{equation}
where $\Phi(\cdot)$ is the Cumulative Distribution Function (CDF) of the  standard Normal distribution and $\sigma>0$ is a scaling parameter. When $\sigma \rightarrow 0$, the CDF converges to an indicator function and \eqref{eq:likelcdf0} reduces to  \eqref{eq:iifu}. For PL, this likelihood was originally proposed by \cite{ChuGhahramani_preference2005} and derived under a Gaussian noise model.  
 
A binary relation on $\mathcal{X} \times \mathcal{X}$ can more in general be  represented through a two-argument function      
$q : \mathcal{X} \times \mathcal{X} \rightarrow \mathbb{R}$              \citep{shafer1974nontransitive,fishburn1988nonlinear}. If ${\bf x}_i$ is in relation with ${\bf x}_j$ then $q({\bf x}_i,{\bf x}_j)>0$. Since in general $q({\bf x}_i,{\bf x}_j)\neq q({\bf x}_j,{\bf x}_i)$ we can equivalently write $q({\bf x}_i,{\bf x}_j)$ as  $q([{\bf x}_i,{\bf x}_j])$, that is as a function of the vector $[{\bf x}_i,{\bf x}_j]$. 
The function $q$ can be interpreted as a ``strength of preference'', with values of $q([{\bf x}_i,{\bf x}_j])$ close to zero indicating a difficult
decision -- Alice cannot distinguish ${\bf x}_i,{\bf x}_j$. This is a natural generalization of representation results for consistent preferences discussed previously, in which case
one can set $q([{\bf x}_i,{\bf x}_j]) = u({\bf x}_i)-u({\bf x}_j)$ for a utility function $u$.

Under this representation,  PL  can be formulated as a classification problem by rewriting
$\mathcal{D}_m$ as the dataset $(X,Y)$: 
$$
X=\begin{bmatrix}
   {\bf x}^{(1)}_l & {\bf x}^{(1)}_r\\
   {\bf x}^{(2)}_l & {\bf x}^{(2)}_r\\
   \vdots & \vdots\\
   {\bf x}^{(m)}_l & {\bf x}^{(m)}_r\\
  \end{bmatrix},~~Y=\begin{bmatrix}
  1\\
  1\\
  \vdots\\   
  1\\
  \end{bmatrix}.
$$
Indeed, most of the initial PL methods solved the PL problem as an augmented binary classification problem. The resulting classification function is not guaranteed to satisfy  asymmetry and negative transitivity in general. However, for kernel-based methods, it is possible to derive classifiers that satisfy one or both these properties. 

Indeed, a GP prior on the latent utility $u({\bf x})\sim \text{GP}(0,k({\bf x},{\bf x}'))$ induces a GP prior on $q([{\bf x}_i,{\bf x}_j]) = u({\bf x}_i)-u({\bf x}_j)$ by linearity \citep{houlsby2011bayesian}: $q([{\bf x}_i,{\bf x}_j]) \sim \text{GP}(0,k_p([{\bf x}_i,{\bf x}_j],[{\bf x}'_i,{\bf x}'_j])$, where  
\begin{align}
\label{eq:kernpref}
 k_p([{\bf x}_i,{\bf x}_j],[{\bf x}'_i,{\bf x}'_j])&=k({\bf x}_i,{\bf x}'_i)-k({\bf x}_i,{\bf x}'_j)-k({\bf x}'_i,{\bf x}_j)+k({\bf x}_j,{\bf x}'_j),
\end{align}
which is called \textit{preference kernel}.  Functions $q$ sampled from the above GP satisfy asymmetry and negative transitivity, and so do the  GP classifier based on it. In the following, we refer to the GP PL-model based on this kernel as  \textit{Preferential GP} (PGP). 

\cite{pahikkala2010learning} instead, using a feature map view, derived a kernel
\begin{align}
 \label{eq:kernprefintr}
 k_a([{\bf x}_i,{\bf x}_j],[{\bf x}'_i,{\bf x}'_j])&=k({\bf x}_i,{\bf x}'_i)k({\bf x}_j,{\bf x}'_j)-k({\bf x}_i,{\bf x}'_j)k({\bf x}'_i,{\bf x}_j),
\end{align}
satisfying asymmetry but not negative transitivity in general. This kernel is known as  \textit{intransitive preference kernel}. A PL model which employs a GP prior on $q$ with kernel \eqref{eq:kernprefintr} has been recently proposed by \cite{chau2022learning}. In the following, we refer to this model as GPGP.

\subsection{Choice functions}
The PL models discussed in the previous section assume that Alice  gives her preferences between pairs of objects. In many situations, she will be instead choosing out of a set with more than two elements. In this more general case, Alice's choices can be formalised through the concept of choice functions.  Let  $\Qset$ denote the set of all  finite subsets of $\mathcal{X}$, then \cite{kreps1990course}:
\begin{definition}
 A choice function $C$ is a set-valued operator on sets of objects. More precisely, it is a map $C: \Qset \rightarrow \Qset$ such that, for any set of objects $A \in \Qset$, the corresponding value of $C$ is a subset $C(A)$ of $A$. 
\end{definition}
It will be assumed throughout this paper that Alice is
able to find a choosable object in every set she is presented with, and therefore $C(A) \neq \emptyset$ for all $A$. It is convenient to introduce the set of rejected objects, denoted by $R(A)$, and  equal to $A\backslash C(A)$.

There are two main interpretations of choice functions. 
In both interpretations, for a given object set $A \in \Qset$, the statement that an object ${\bf x}_j \in A$ is rejected from $A$ (that is, ${\bf x}_j  \notin C(A)$) means that there is at least one object  ${\bf x}_i \in A$ that Alice strictly prefers over ${\bf x}_j$. Note that, Alice is not required to tell us which object(s) in $C(A)$ she strictly prefers to ${\bf x}_j$. This makes choice functions a very easy-to-use tool to express choices. The two interpretations differ instead in the meaning of the statement  ${\bf x}_i  \in C(A)$.
\begin{enumerate}
    \item In the traditional interpretation \cite{kreps1990course}, one reads the statement ${\bf x}_i  \in C(A)$ as ``${\bf x}_i$ is considered to be at least as good as all other objects in $C(A)$,” and thus infers from a statement like $\{{\bf x}_i, {\bf x}_j\} \subseteq C(A)$ that Alice is indifferent between ${\bf x}_i$ and ${\bf x}_j$.
    \item The alternative interpretation of  $\{{\bf x}_i, {\bf x}_j\} \subseteq C(A)$ is that ${\bf x}_i$ and ${\bf x}_j$ are incomparable for Alice. 
\end{enumerate} 
Incomparability can arise for two reasons.
First, the objects to be compared have multiple utilities for Alice. For example, if the objects are laptops, Alice may consider multiple utilities such as speed and weight.
Second, incomparability can arise due to incompleteness     \citep{seidenfeld2010coherent}, which represents simply an absence of knowledge about the
underlying utility function.
We can model both these cases assuming there are multiple utility functions (due either to incomparability and incompleteness) and then interpret the statement $\{{\bf x}_i, {\bf x}_j\} \subseteq C(A)$ as ``${\bf x}_i$ and ${\bf x}_j$ are undominated in $A$ in a Pareto sense''.\footnote{As for the case of preferences, a choice function must satisfy some consistency properties to be Pareto rationalizable  \citep{moulin1985choice,eliaz2006indifference}.}

 This approach was originally proposed in \cite{pfannschmidt2020learning} to learn choice functions.  The authors devise a differentiable loss function based on two hinge loss terms. Furthermore, they add two additional terms to the loss function: (i) an $L^2$ regularization term; (ii) a multidimensional scaling (MDS) loss to ensure that objects close to each other in the inputs space $\mathcal{X}$ will also be close in the embedding space $\mathbb{R}^{d}$. Overall the loss function is the sum of four terms weighted by four non-negative scalar parameters $\alpha_1,\alpha_2,\alpha_3,\alpha_4$ which sums up to one.
These weights are treated as hyperparameters of the learning algorithm.
This loss function is then used to learn a (deep) multi-layer perceptron to represent the embedding. We refer to this model as \textit{ChoiceNN}.  In the next section, we instead propose a GP model to learn choice functions from choice data. In Section \ref{sec:vsNN}, we will show that the GP-based model outperforms \textit{ChoiceNN}.

 Finally, it is worth to mention that PL with more than two objects was also considered by \cite{Siivola21}, the so-called batch-preference model. This model considers the case where a subject expresses  preferences for a group of objects. However, the batch-preference model in \citep{Siivola21}  assumes that two objects are always comparable and, therefore, as we will show in Section \ref{sec:otherprefmodels}, this model  assumes a single utility function. \vspace{-1cm}

\section{Methodology}
For each $A$, we interpret $C(A)$ as the \textit{undominated set} in the \textit{strong Pareto sense} with $R(A)$ being the set of dominated objects. In other words, we assume that there is a latent vector function ${\bf u}({\bf x})=[u_1({\bf x}),\dots,u_{d}({\bf x})]^\top$, for some finite dimension $d$, which embeds the objects ${\bf x}$ into a  space $\mathbb{R}^{d}$. The choice set can then be represented through a Pareto set of strongly undominated objects:
\begin{align}
\label{eq:likcondpareto1}
 &\neg  \left( \min_{i \in\{1,\dots,d\}} (u_i({\bf o})-u_i({\bf v}))< 0, ~\forall {\bf o} \in C(A)\right),\forall {\bf v} \in R(A),\\
  \label{eq:likcondpareto2}
 &\min_{i \in\{1,\dots,d\}} (u_i({\bf o})-u_i({\bf v}))< 0, ~\forall {\bf o},{\bf v}\in C(A), ~ {\bf o} \neq {\bf v}.
 \end{align}
 Condition \eqref{eq:likcondpareto1} means that, for each object ${\bf v} \in R(A)$, it is not true ($\neg$ stands for logical negation) that all objects in $C(A)$ are worse than ${\bf v}$, i.e. there is at least an object in $C(A)$ which is not worse than ${\bf v}$. Condition   \eqref{eq:likcondpareto2} means that, for each object in $C(A)$, there is no better object in $C(A)$. This requires that the latent functions values of the objects should be consistent with the choice function implied relations. 
 
To account for errors in Alice's choices, we extend the likelihood in \eqref{eq:likelcdf0} to choice functions. Consider the vectors $X=[\bx_1,\bx_2,\dots,\bx_t]^\top$ with $\bx \in \mathcal{X}$,  
$
{\bf u}({\bf x}_i)=[
u_1({\bf x}_i),u_2({\bf x}_i),\dots, u_d({\bf x}_i)]$ and ${\bf u}(X)=[{\bf u}({\bf x}_1), {\bf u}({\bf x}_2),\dots, {\bf u}({\bf x}_t)]^\top$,
and the choice dataset 
$$
\mathcal{D}_m=\{(C(A_s),A_s): \text{ for } s=1,\dots,m\},$$
where $A_s\subset X$ for each $s$. The likelihood is defined as
\begin{equation}
  \label{eq:likelihoodexpanse0}
   \begin{aligned}
  &p(\mathcal{D}_m|{\bf u}(X))=\prod_{k=1}^m p(C(A_k),A_k|{\bf u}(X))\\
   &=\prod_{k=1}^m \prod\limits_{\{{\bf o},{\bf v}\} \in C_2(A_k)}\Bigg( 1-\prod_{i=1}^d \Phi\left(\frac{u_i({\bf o})-u_i({\bf v})}{\sigma}\right)\\
   &~~~~~~~~~~~~~~~~~~~~~~~~~~~~~~~~~-\prod_{i=1}^d \Phi\left(\frac{u_i({\bf v})-u_i({\bf o})}{\sigma}\right)\Bigg)\\
         &\prod_{{\bf v} \in R(A_k)}\Bigg(1- \prod_{{\bf o} \in C(A_k)} \left(1- \prod_{i=1}^d \Phi\left(\frac{u_i({\bf o})-u_i({\bf v})}{\sigma}\right)\right)\\
   \end{aligned}    
 \end{equation}
 where the notation $\{{\bf o},{\bf v}\} \in C_2(A_k)$ means that the pair $\{{\bf o},{\bf v}\}$ is an element of $C_2(A_k)$, which is the set   of all possible 2-combination (without repetition) of the elements of the set $C(A_k)$.
   The product in the first and second row in \eqref{eq:likelihoodexpanse0} is a probabilistic relaxation of \eqref{eq:likcondpareto2}.
  The product in the last row in \eqref{eq:likelihoodexpanse0} is a probabilistic relaxation of \eqref{eq:likcondpareto1}. In Appendix \ref{app:like}, we discuss how to vectorise this complex likelihood.

\paragraph{Prior:} Similarly to GP processes for multiclass classification \cite{williams1998bayesian}, we model each  latent utility function in the vector ${\bf u}({\bf x})=[u_1({\bf x}),\dots,u_{d}({\bf x})]^\top$ as an independent GP:
\begin{equation}
\label{eq:prior}
 u_i({\bf x}) \sim \text{GP}_i(0,k_i({\bf x},{\bf x}')), ~~~~i=1,2,\dots,d.
\end{equation}
 Each GP is fully specified by its kernel function $k_i(\cdot,\cdot)$, which defines the covariance of the latent function between any two points. The model parameters are the kernel parameters (lengthscales) in $k_i(\cdot,\cdot)$, and the scale parameter $\sigma$ in the likelihood function. These parameters can be collected into a hyperparameter vector $\boldsymbol{\theta}$. 
 

\subsection{Posterior and prediction}
\label{sec:posterior}
The posterior probability of ${\bf u}(X)$ is
\begin{equation}
\label{eq:post}
p({\bf u}(X)|\mathcal{D}_m)=\frac{p({\bf u}(X))}{p(\mathcal{D}_m)} \prod_{k=1}^m p(C(A_k),A_k|{\bf u}(X)),
\end{equation}
where the prior over the component of ${\bf u}$ is defined in \eqref{eq:prior}, the likelihood is defined in \eqref{eq:likelihoodexpanse0}
and the probability of the evidence is $p(\mathcal{D}_m)= \int p(\mathcal{D}_m|{\bf u}(X))p({\bf u}(X)) d{\bf u}(X)$. The posterior $p({\bf u}(X)|\mathcal{D}_m)$ is intractable because it is not a GP. Contrarily to the case of binary preferences, the posterior is not a Skew Gaussian Process \citep{benavoli2020skew,benavoli2021}.    Inference on ${\bf u}$ could be computed using approximation methods such as (i) the Laplace Approximation (LA)  \cite{mackay1996bayesian}; (ii) Variational Inference (VI) \citep{opper2009variational,hensman2015scalable}.
  
As discussed  in Appendix \ref{app:LA}, LA cannot be applied due to the so-called `label switching' problem. Therefore, we resort to VI to learn at the same time the hyperparameters $\boldsymbol{\theta}$ of the kernel and a Gaussian approximation of the posterior  $p({\bf u}|\mathcal{D}_m)$.\footnote{We implemented our model using automatic-differentiation in Jax   \citep{jax2018github}. Details  are reported in Appendix \ref{app:VI}.}      

\paragraph{Prediction and Inferences}
Let $X^*=\{{\bf x}^*_1,\dots,{\bf x}^*_p\}$ be a set including $p$ test points and ${\bf u}(X^*)=[{\bf u}({\bf x}_1^*),\dots,{\bf u}({\bf x}_m^*)]^\top$. Under the GP prior assumption on $u$, the conditional predictive distribution $p({\bf u}(X^*)|{\bf u}(X))$ is Gaussian and, therefore, 
\begin{equation}
\label{eq:pred}
p({\bf u}(X^*)|\mathcal{D}_m)= \int p({\bf u}(X^*)|{\bf u}(X))p({\bf u}(X)|\mathcal{D}_m) d{\bf u}(X)
\end{equation}
 can be easily computed analytically using the VI posterior $p({\bf u}|\mathcal{D}_m)$, which is Gaussian.
In choice function learning, we are interested in the inference: 
\begin{equation}
\label{eq:infer}
\begin{aligned}
 P(C(A^*),A^*|\mathcal{D}_m)=\int &p(C(A^*),A^*|{\bf u}(X^*)) \\
 &p({\bf u}(X^*)|\mathcal{D}_m) d{\bf u}(X^*),
\end{aligned}
\end{equation}
which returns the posterior probability that the agent chooses the options $C(A^*)$ from the set of options $A^*$.
This probability can be easily computed via Monte Carlo sampling from the approximate posterior $p({\bf u}(X^*)|\mathcal{D}_m)$, which is Gaussian.

\begin{example}
\label{ex:1}
We illustrate the overall model with an example. We consider the bi-dimensional utility function ${\bf u}(x)=[\cos(2x),-sin(2 x)]$ with $x \in \mathbb{R}$.

\begin{tabular}{c}\vspace{-0.3cm} 
~~~~~~~~~~~~~\includegraphics[height=3cm]{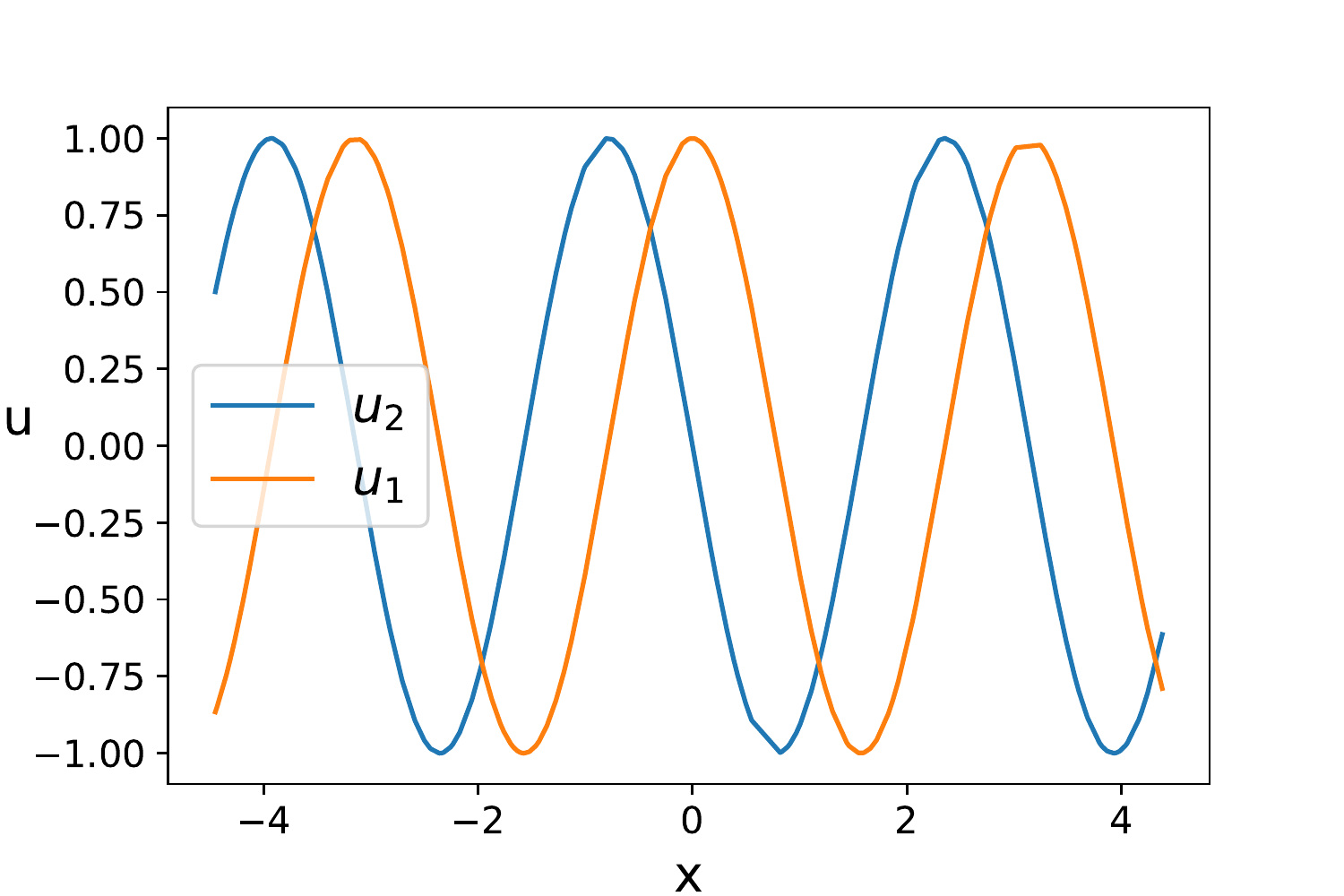}
\vspace{0cm}                                                                                                                                           \end{tabular}

We use ${\bf u}$ to define a choice function. For instance, consider the set of options $A_k=\{0,0.5,2.36\}$, given that
${\bf u}( 0)=[1,0]$, ${\bf u}(0.5)=[0.54,-0.84]$, $
{\bf u}( 2.36)=[0,1]$,
we have that $C(A_k)=\{0,2.36\}$ and $R(A_k)=A_k \backslash C(A_k)=\{0.5\}$. In fact, one can notice that  $[1,0]$ dominates  $[0.54,-0.84]$ on both the utilities, and $[1,0]$  and $[0,1]$ are incomparable.  We sample $200$ inputs $x_i$ at random in $[-4.5,4.5]$ and, using the above approach, we generate
\begin{itemize}
 \item  $m=50$ random subsets $\{A_k\}_{k=1}^m$ of the 200 points each one of size $|A_k|=3$ (respectively $|A_k|=5$) and computed the corresponding choice pairs $(C(A_k),A_k)$ based on  ${\bf u}$;
    \item  $m=150$ random subsets $\{A_k\}_{k=1}^m$ each one of size $|A_k|=3$ (respectively $|A_k|=5$) and computed the corresponding choice pairs $(C(A_k),A_k)$ based on  ${\bf u}$;
\end{itemize}
 Fixing the latent dimension $d=2$, we use these datasets to compute the posterior means and $95\%$ credible intervals of the latent functions learned using the model introduced in Section \ref{sec:posterior}.
The four posterior plots are shown in Figure \ref{fig:post}. By comparing the 1st with the 3rd plot and the 2nd with the 4th plot, it can be noticed how the posterior means become more accurate (and the credible interval smaller) at the increase of the size dataset (from m=50 to m=150 choice-sets). 
By comparing the 1st with the 2nd plot and the 3rd with the 4th plot, it is evident that estimating the latent function becomes more complex at the increase of $|A_k|$. The reason is not difficult to understand. Given $A_k$, $R(A_k)$ includes the set of rejected objects. These are objects that are dominated by (at least) one of the objects in $C(A_k)$, but we do not know which one(s). This uncertainty increases with the size of $|A_k|$ and makes the estimation problem more difficult.
\end{example}

\begin{figure*}
	\centering
	\begin{tabular}{cccc}
		\includegraphics[height=3cm,trim={0.8cm 0.0cm 1.5cm 0.0cm }, clip]{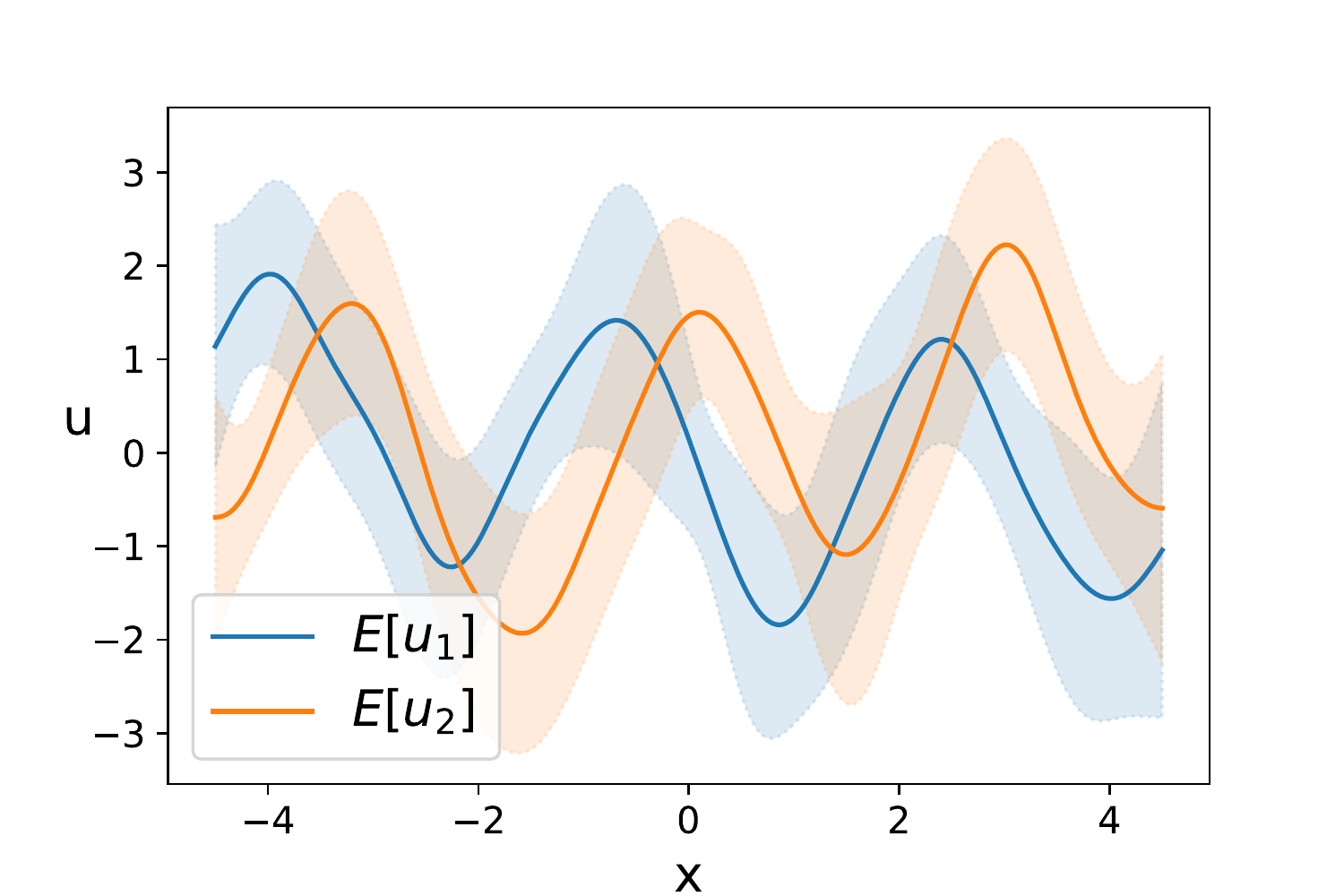} &
		\includegraphics[height=3.0cm,trim={0.8cm 0.0cm 1.5cm 0.0cm }, clip]{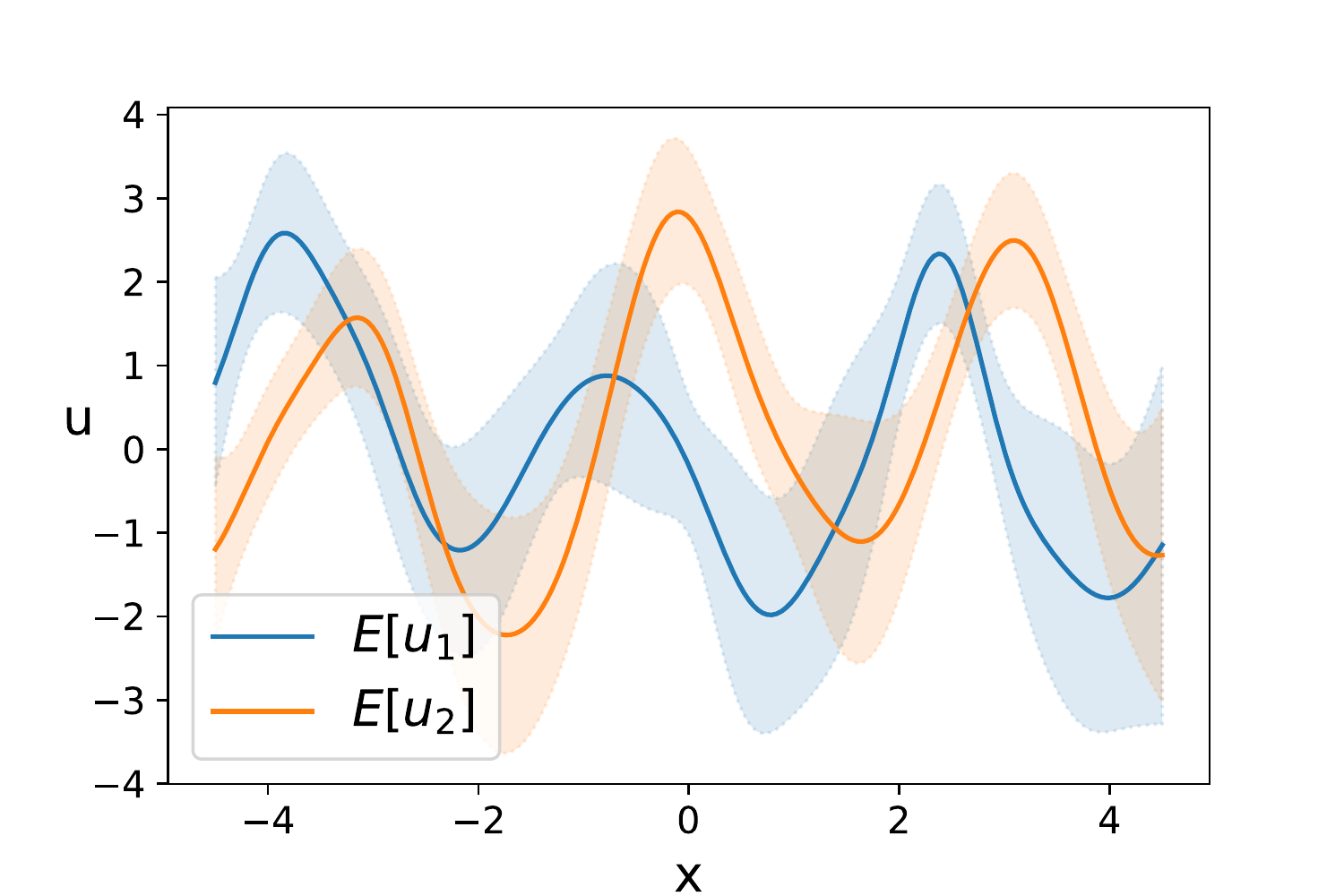} &
			\includegraphics[height=3.0cm,trim={0.8cm 0.0cm 1.5cm 0.0cm }, clip]{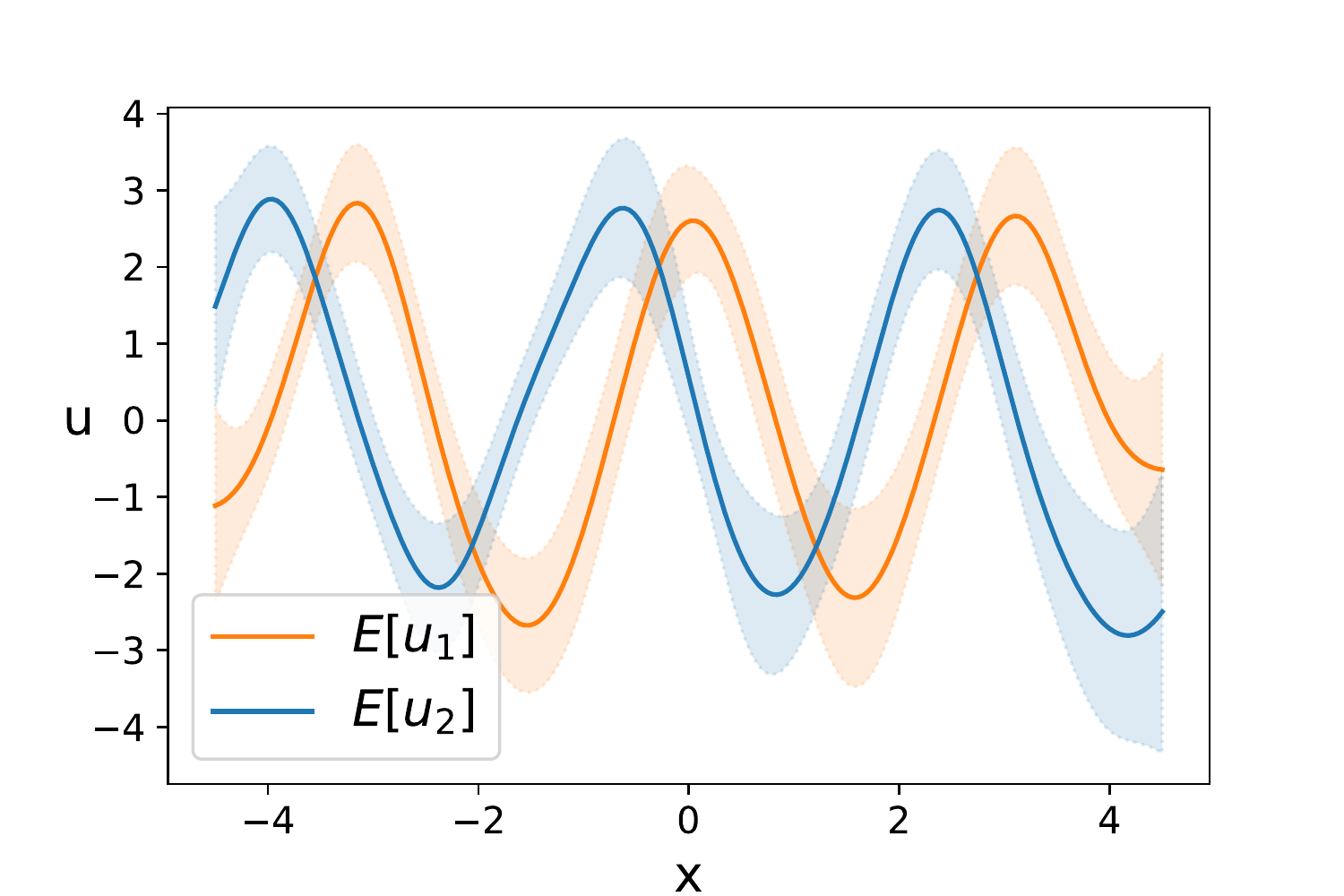} &
			\includegraphics[height=3.0cm,trim={0.8cm 0.0cm 1.5cm 0.0cm }, clip]{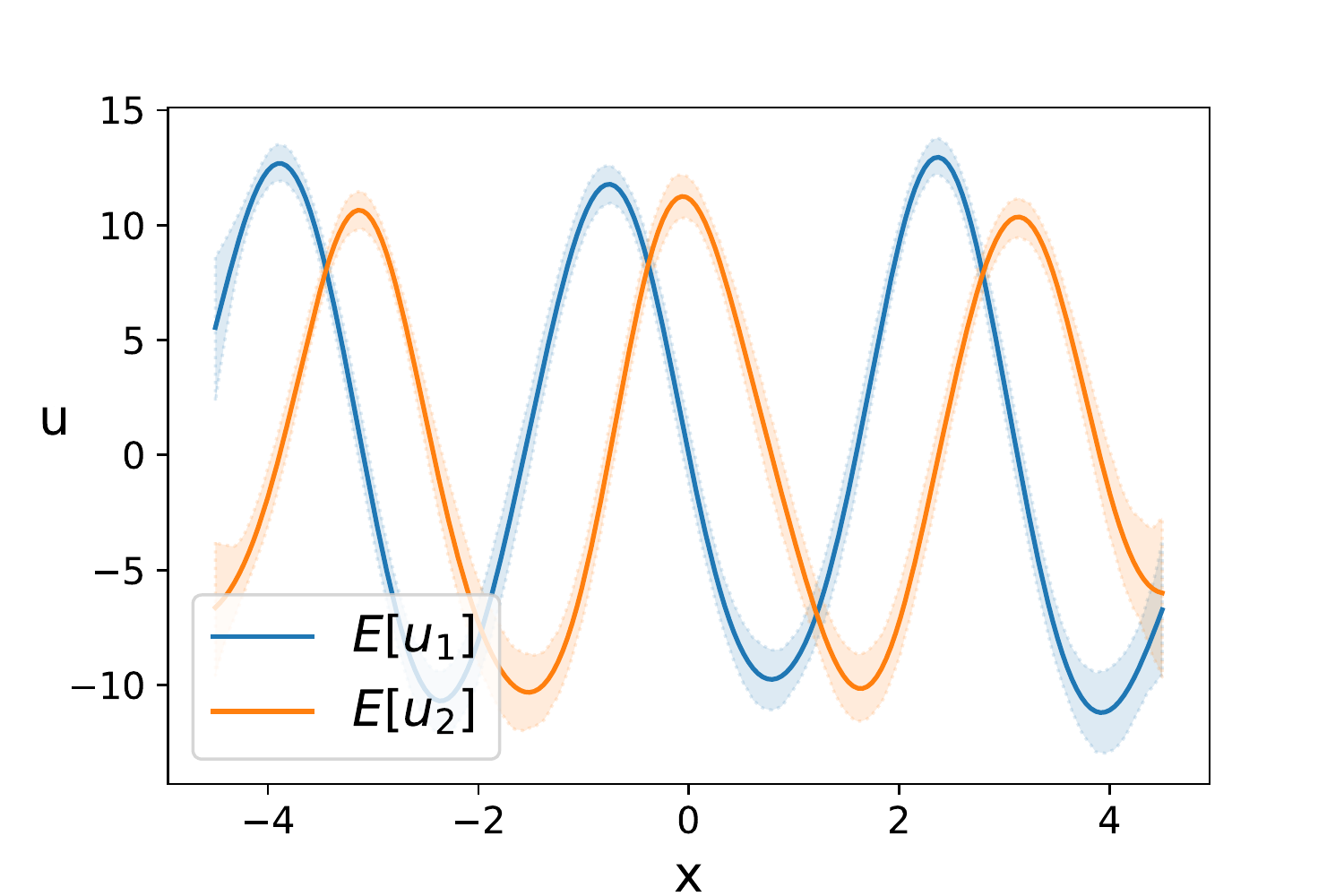} \\  
			$N=50,~|A_k|=3$ & $N=50,~|A_k|=5$ & $N=150,~|A_k|=3$ & $N=150,~|A_k|=5$\\ 
	\end{tabular}  
	\caption{Posterior mean and $95\%$ credible intervals of the two latent functions for the four artificial datasets.}
	\label{fig:post}
\end{figure*}  

\paragraph{Scalability:} 
The computational complexity for ChoiceGP is similar to that in GP multiclass classification. By exploiting the independence structure of the prior in the VI, we need
 storing and inverting $d$  kernel matrices with dimension $t \times t$. For large $t$, there are a  number of well established ways to scale up GPs that can be applied to ChoiceGP \citep{quinonero2005unifying,snelson2006sparse,pmlrv5titsias09a,Hensman2013,hernandez2016scalable,bauer2016understanding,SCHURCH2020,schuch2023correlated}.   
 
\subsection{Latent dimension selection}
\label{sec:lat}
In the previous sections, we provided a GP-based model  to learn choice functions. We refer to this model as \textit{ChoiceGP}$_d$.  \textit{ChoiceGP}$_d$ is conditional on the pre-defined latent dimension $d$ (that is, the dimension of the vector of the latent functions ${\bf u}({\bf x})=[u_1({\bf x}),\dots,u_{d}({\bf x})]^\top$). Although, it is sometimes reasonable to assume  the number of utility functions defining the choice function is known, it is crucial to to derive a statistical method to select $d$ from data.   

We propose a forward selection method. We start learning the model \textit{ChoiceGP}$_1$ and we increase the dimension $d$ in a stepwise manner (so learning \textit{ChoiceGP}$_2$,\textit{ChoiceGP}$_3$ and so on) until some model selection criterion is optimised. Criteria like AIC and BIC are inappropriate for the proposed GP-based choice  function model, since its nonparametric nature implies that the number of parameters increases also with the size of the data (as $d \times m$). We propose to use instead the \textit{Pareto Smoothed Importance sampling Leave-One-Out} cross-validation (PSIS-LOO) \cite{vehtari2017practical}. Exact cross-validation requires re-fitting the model with different training sets. Instead, PSIS-LOO can be computed efficiently using the samples from the posterior. 

We define the Bayesian LOO estimate of out-of-sample predictive fit for the model in \eqref{eq:post}:
\begin{equation}
\label{eq:bloo}
 \varphi=\sum_{k=1}^m \ln p(z_k|z_{-k}),
\end{equation}
where $z_k=(C(A_k),A_k)$, $z_{-k}=\{(C(A_i),A_i)\}_{i=1,i\neq k}^m$,
\begin{equation}
\label{eq:bloo1}
p(z_k|z_{-k})=\int p(z_k|{\bf u}(X))p({\bf u}(X)|z_{-k})d{\bf u}(X).
\end{equation}
As derived in \cite{gelfand1992model}, we can evaluate \eqref{eq:bloo1} using the samples from the full posterior, that is  ${\bf u}^{(s)}(X)\sim p({\bf u}(X)|\{z_k,z_{-k}\})=p({\bf u}(X)|\mathcal{D})$ for $s=1,\dots,S$.\footnote{We generate these samples from the variational posterior.} We first define the importance weights:
$$
w^{(s)}_k=\frac{1}{p(z_k|{\bf u}^{(s)}(X))}\propto \frac{p({\bf u}^{(s)}(X)|z_{-k})}{p({\bf u}^{(s)}(X)|\{z_k,z_{-k}\})}
$$
and then approximate \eqref{eq:bloo1} as:
\begin{equation}
\label{eq:bloo2}
p(z_k|z_{-k})\approx \frac{\sum_{s=1}^S w^{(s)}_k p(z_k|{\bf u}^{(s)}(X))}{\sum_{s=1}^S w^{(s)}_k }.
\end{equation}
It can be noticed that \eqref{eq:bloo2} is a function of $p(z_k|{\bf u}^{(s)}(X))$ only, which can easily be computed from the posterior samples. Unfortunately,  a direct use of \eqref{eq:bloo2} induces instability because the importance
weights can have high variance. To address this issue, \cite{vehtari2017practical} applies a simple smoothing
procedure to the importance weights using a Pareto distribution. We provide an example hereafter.

\begin{example}
\label{ex:1}
We run the latent-dimension selection procedure on the four datasets 
in Example \ref{ex:1}. The below table reports the PSIS-LOO for different values of the dimension $d$. It can be observed how the selection procedure always selects the true dimension $d=2$.

 \begin{tabular}{c|cc|cc}   
 & \multicolumn{2}{c|}{{$m=50$}} & \multicolumn{2}{c}{{$m=150$}}\\
    {$d$} & {$|A_k|=3$} & {$|A_k|=5$} &  {$|A_k|=3$} & {$|A_k|=5$}  \\ \midrule
    1  & -882 & -1906 & -3213 &  -6108 \\
    2  & -34  & -118 & -69  &  -84  \\
    3  & -42  & -134 & -80  & -95 \\
    4  & -50  & -152 & -91  &  -109   \\  \bottomrule
\end{tabular}
\end{example}    
In Section \ref{sec:real}, we will show that the proposed PSIS-LOO-based forward procedure also works  on real datasets.

\subsection{Relation to (batch-)preference}
\label{sec:otherprefmodels}
For $d=1$ (the latent dimension is one), we have that $|C(A_k)|=1$. This means the subject always selects a single best object. In this case, the likelihood \eqref{eq:likelihoodexpanse0}, for a given $k$, simplifies to
  \begin{equation}
  \label{eq:likelihood1}
   \begin{aligned}
  &p((C(A_k),A_k)|{\bf u}(X))= \prod_{{\bf o} \in C(A_k)}\prod_{{\bf v} \in R(A_k)}\Phi\left(\frac{u({\bf o})-u({\bf v})}{\sigma}\right),
   \end{aligned}  
 \end{equation}
and reduces to the likelihood \eqref{eq:likelcdf0} when $|R(A_k)|=1$ (that is, in the binary case $|A_k|=2$). For $|A_k|>2$, the likelihood \eqref{eq:likelihood1} is a lower bound of the batch preference likelihood derived in \cite[Eq.3]{Siivola21}:
    \begin{equation}
  \label{eq:likelihoodbatch}
 \begin{aligned}
       \int &\left(\prod_{{\bf v} \in R(A_k)}\Phi\left(\frac{u({\bf o})+w_k-u({\bf v})}{\sigma}\right)\right)N(w_k;0,\sigma^2)dw_k,\\
   \end{aligned}
 \end{equation}
as proven in Appendix \ref{app:batch}.  The difference between \eqref{eq:likelihood1} and \eqref{eq:likelihoodbatch} comes from two different ways of modelling errors. The likelihood \eqref{eq:likelihoodbatch}  assumes that inconsistencies in Alice's preferences are due to an additive Gaussian noise perturbation of the true utility. The likelihood \eqref{eq:likelihood1} instead assumes that inconsistencies are due to a limit of discernability, that is the probability of error is inversely proportional to the difference between the utilities of the two objects to be compared (when this difference is  zero Alice has 50\% chance to select one object or another). Appendix \ref{app:batch} includes a further discussion about these two likelihoods.

\begin{figure*}
	\centering
  \includegraphics[height=4cm,trim={0.0cm 0.0cm 1.5cm 0.0cm }, clip]{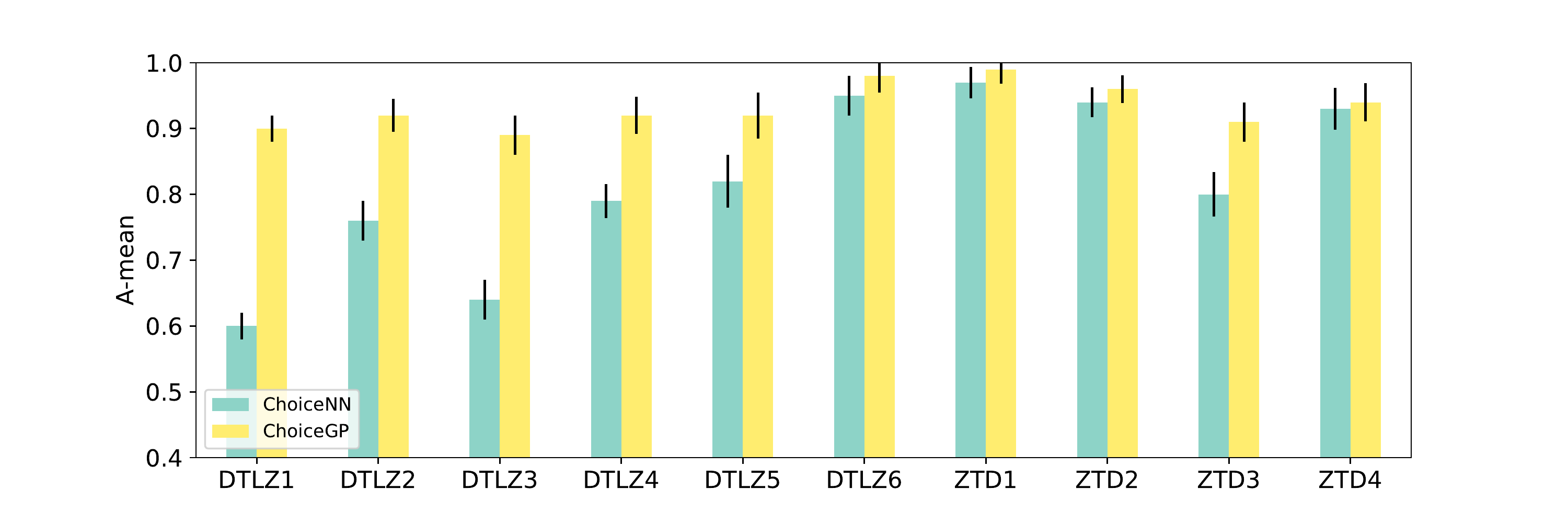}
	\caption{Comparisons ChoiceGP vs.\ ChoiceNN on 10 multi-criteria optimization problems. A-means are
averaged over 5 runs and error bars of 1 standard deviation are provided.}
	\label{fig:benchmarks}
\end{figure*}

\section{Experiments}
Our experiments aim to compare ChoiceGP with the state-of-the-art methods for choice functions and preference learning.
In section \ref{sec:vsNN}, we compare ChoiceGP
 against ChoiceNN \citep{pfannschmidt2020learning} on choice data simulated using multi-utility functions taken from
benchmark problems used in multi-criteria optimization.
 In section \ref{sec:gpgp}, using simulated preferences, we compare \textit{ChoiceGP} with \textit{Preferential GP} (PGP) \citep{ChuGhahramani_preference2005}, \textit{General Preferential GP} (GPGP) \citep{chau2022learning} and GP with data augmentation (PairGP) \citep{chau2022learning}.  PairGP solves the PL problem as an augmented binary classification problem.
In PairGP, skew-symmetry is further enforced by averaging the model outputs \citep[Sec.\ 3.3]{chau2022learning}. 
For PGP, GPGP and PairGP, we use the Laplace approximation to compute an approximation of the posterior. Finally in Section \ref{sec:gpgp}, we compare ChoiceGP, PGP, GPGP and PairGP using real-world datasets. For all
methods involving kernels, we use the Gaussian radial
basis function kernel with automatic relevance determination (ARD):
$k({\bf x},{\bf x}')=\exp\left(-\sum_{i=1}^c \frac{({x}_i-{x}_i')^2}{2 \ell_i^2}\right)$. The scale-parameter of the kernel is set to one, but instead we estimate the scaling parameter $\sigma$ of the likelihood.

\subsection{Benchmark optimisation problems}
\label{sec:vsNN}

 \paragraph{Data generation}
 In this section, we repeat  the experiment  in \citep[Sec. 4]{pfannschmidt2020learning}. Choice data are simulated using multi-utility functions taken from
benchmark problems used in multi-criteria optimization: the DTLZ test suite \citep{deb2005scalable} and the ZDT test suite \citep{zitzler2000comparison}, for a total of 10 benchmarks.
 In all experiments, the dimension of the choice-set is $|A|=10$ and each object ${\bf x}_i \in \mathbb{R}^6$ (6 features). A total of 40,960 choice sets are generated.  For the
DTLZ problems, the number of objective functions is set to 5 (i.e., $d=5$). As performance, \cite{pfannschmidt2020learning} used the average \textit{A-mean} (A-mean is the arithmetic mean of the true positive and true negative rate) of 5 repetitions of a Monte Carlo cross validation with a 90/10\% split into training and test data.
For ChoiceNN, the training instances are further split into 1/9 validation instances and 8/9
training instances in order to optimize the hyperparameters: (a) the loss weights $\alpha_1,\alpha_2,\alpha_3,\alpha_4$ (b) the number of hidden units and layers, using 60 iterations of Bayesian optimization. For both ChoiceNN and ChoiceGP, the latent dimension is equal to the number of objective functions.

\paragraph{Results} Figure \ref{fig:benchmarks} reports  the average A-mean of the two models when
predicting choices on held-out data.
ChoiceGP significantly outperforms ChoiceNN. We have found that this is due to  ChoiceNN not often being able to find latent functions that are consistent with the training data. The disadvantage of a parametric method, like ChoiceNN, is  that the latent utility functions depend nonlinearly on the optimisation parameters. Instead, in ChoiceGP, the values of the utility functions at the training data are part of the  the variational parameters and, therefore, can be more easily optimised.
We provide a simple example in  Appendix \ref{sec:counterexample} to illustrate this issue. For this reason, we have not included  ChoiceNN in the subsequent experiments.

\begin{figure*}
	\centering
	\begin{tabular}{ccc}
		\includegraphics[height=3cm,trim={0.0cm 0.0cm 1.5cm 0.0cm }, clip]{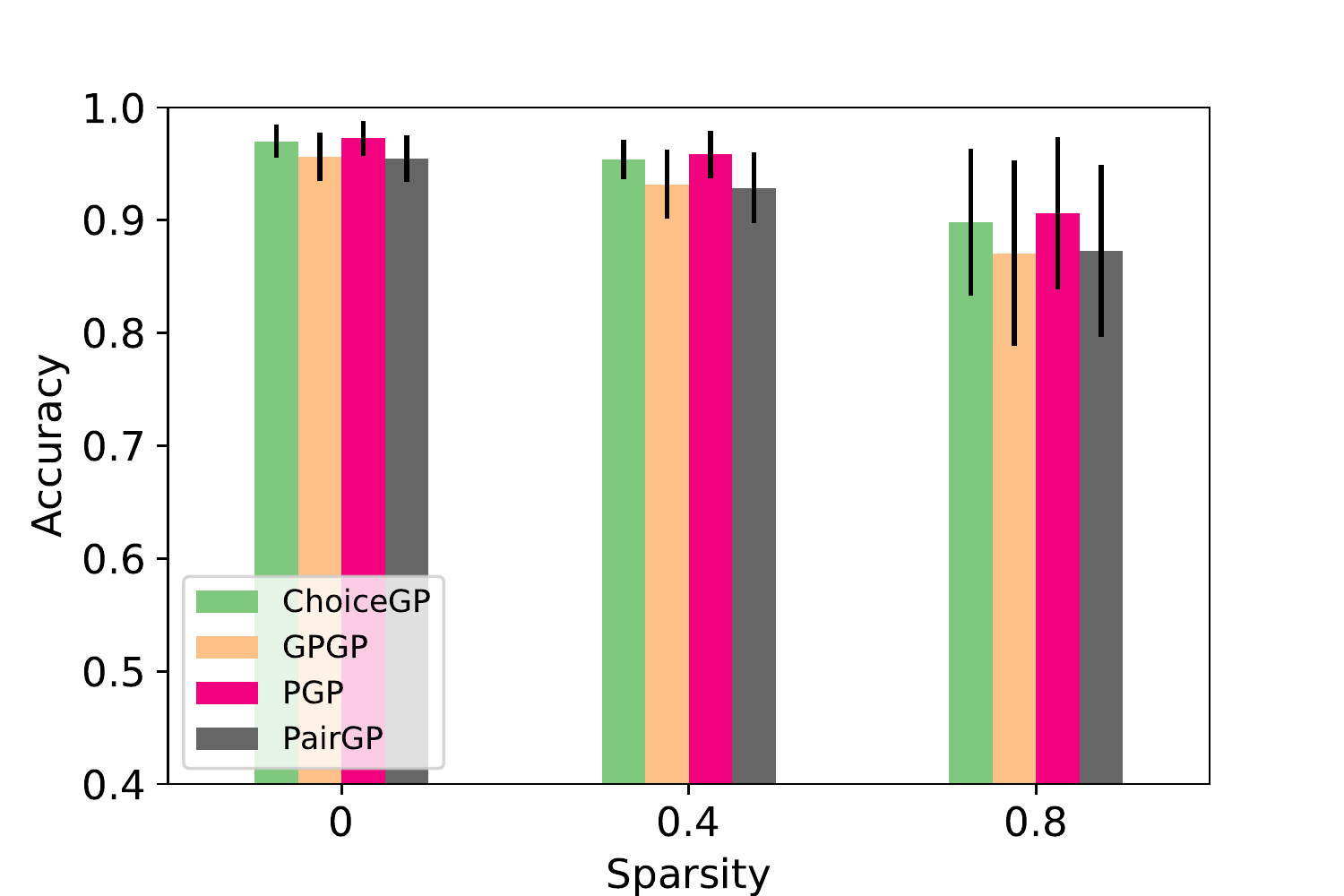} &
		\includegraphics[height=3.0cm,trim={0.0cm 0.0cm 1.5cm 0.0cm }, clip]{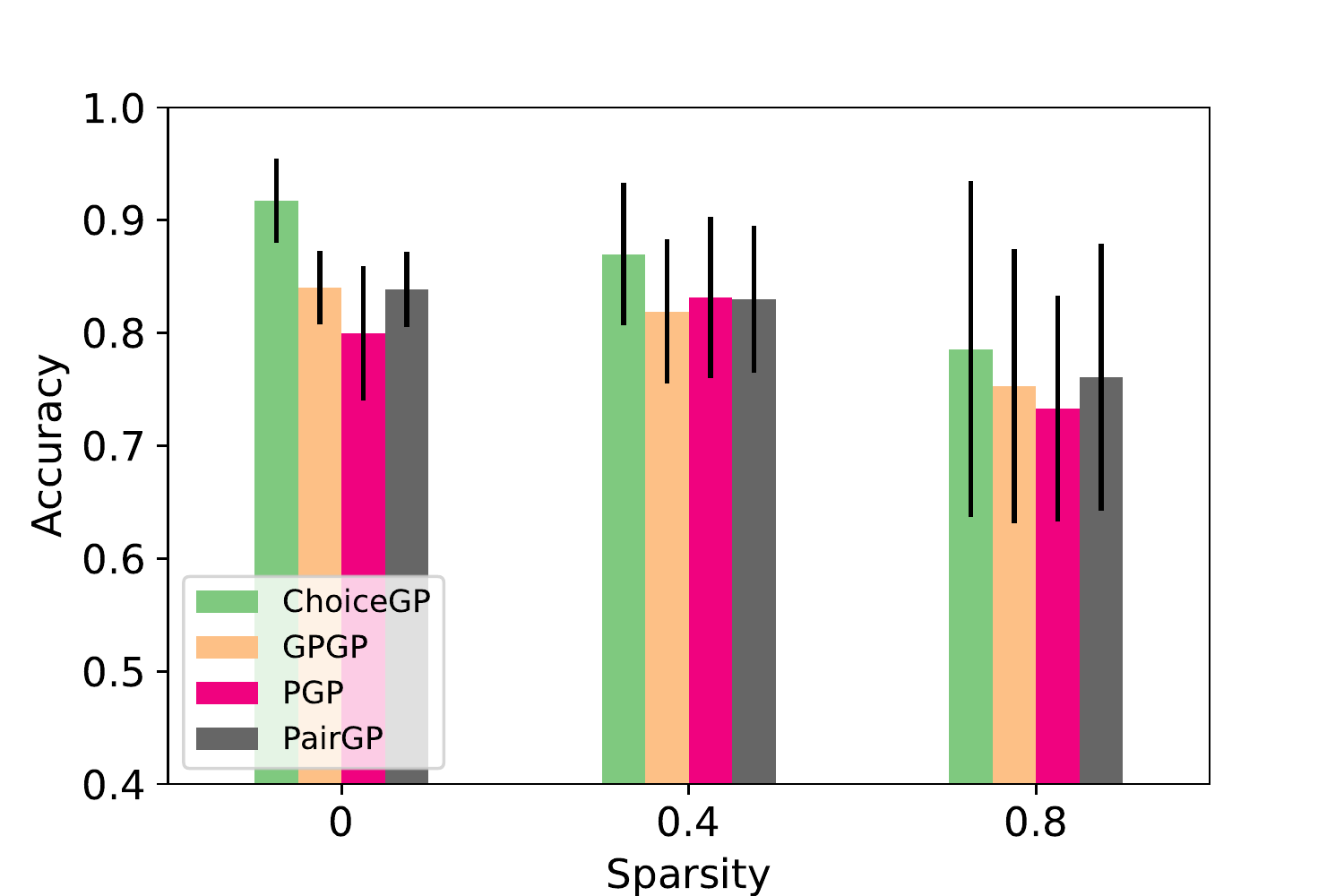} &
		\includegraphics[height=3.0cm,trim={0.0cm 0.0cm 1.5cm 0.0cm }, clip]{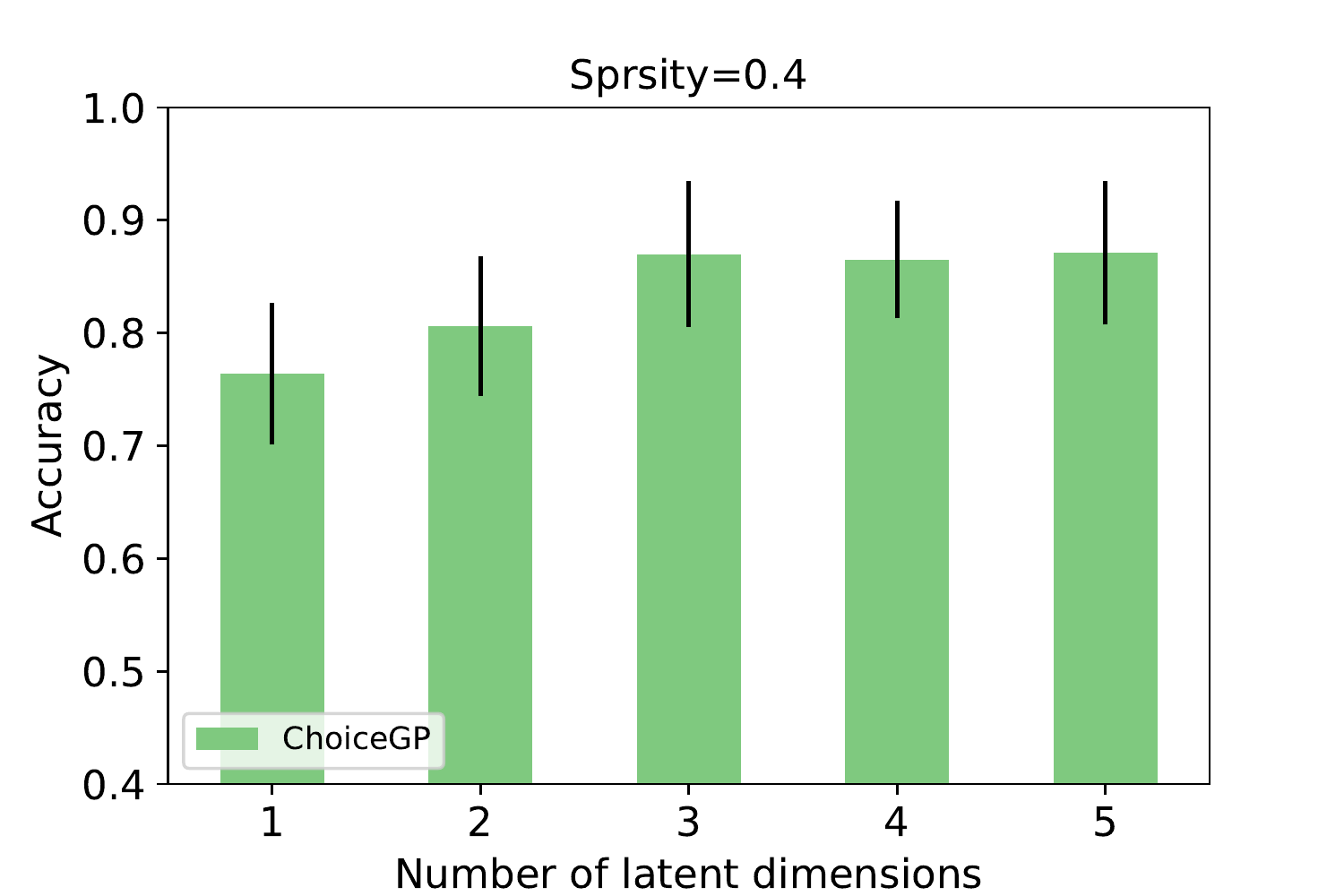} \\
		L=1 ($d=1$) & L=2 ($d=3$) & L=2 ($d=1,2,3,4,5$)\\
	\end{tabular}  
	\caption{Comparisons of algorithms for simulations at different sparsity and inconsistency level. Accuracies are
averaged over 20 runs and error bars of 1 standard deviation are provided}
	\label{fig:rescycle}
\end{figure*}  

\subsection{Inconsistent
preferences}
\label{sec:gpgp}  
 \paragraph{Data generation}
In this section, we repeat the experiment proposed in \citep[Sec. 4.1]{chau2022learning}.  Consider a data matrix $X \in \mathbb{R}^{n \times c}$. We randomly assign to each row a latent variable $z \in \{1,\dots,L\}$ and generate a set of utility functions $\{u_{z,z'}\}_{z,z'=1}^L$ , i.e. there is
a different utility function for each pair $z,z'$ of latent states. We impose the constraint  $u_{z,z'}=u_{z',z}$ and assume that
$u_{z,z'}({\bf x})=\sum_{j=1}^n \alpha_{j}^{z,z'}k({\bf x},{\bf x}_j)$
with each vector $\alpha_{j}^{z,z'}\stackrel{\text{i.i.d.}}{\sim} N(0, I_n)$.  

We generate two datasets. In the first dataset, as done in  \citep{chau2022learning}, the comparison between objects ${\bf x}_i,{\bf x}_j$ is  conducted based on the utility selected
by their latent states, i.e. ${\bf x}_i \succ {\bf x}_j$ iff $u_{z_i,z_j}({\bf x}_i)>u_{z_i,z_j}({\bf x}_j)$, leading to a dataset of $m$ pairwise comparisons:
$$
 \mathcal{D}^{(1)}_m = \{ {\bf x}^{(s)}_l \succ {\bf x}^{(s)}_r:~~ s = 1,\dots,m\}.
 $$
$\mathcal{D}^{(1)}_m$ includes preferences violating negative transitivity.  In the second dataset, the comparison between objects ${\bf x}_i,{\bf x}_j$ is  based on  Pareto's dominance criterion, i.e. ${\bf x}_i$ is chosen from the set $A=\{{\bf x}_i, {\bf x}_j\}$ iff $u_{z,z'}({\bf x}_i)>u_{z,z'}({\bf x}_j)$ for all $z,z'$. This naturally originates into a choice dataset:
$$
 \mathcal{D}^{(2)}_m = \{ (C(A_s),A_s):~~ s = 1,\dots,m\},
 $$   
 where $A_s=\{{\bf x}^{(s)}_l,{\bf x}^{(s)}_r\}$ (the same pairs as in $ \mathcal{D}^{(1)}_m$) and $C(A_s)$ equal to either $\{{\bf x}^{(s)}_l\}$ if $u_{z,z'}({\bf x}^{(s)}_l)>u_{z,z'}({\bf x}^{(s)}_r)$ for all $z,z'$ or to $\{{\bf x}^{(s)}_l,{\bf x}^{(s)}_r\}$ otherwise.   Dataset $\mathcal{D}^{(2)}_m$ corresponds to a scenario in which Alice is allowed to express incomplete judgments while,  in  $\mathcal{D}^{(1)}_m$, Alice is compelled to always choose between two objects.
 

 \paragraph{Results} Figure \ref{fig:rescycle}
 reports  the accuracy for the 4 models when
predicting preferences on held-out data.
In the case $L=1~(d=1)$, ChoiceGP and PGP coincide, since there is only one utility function. They outperform both GPGP and PairGP.
In the case $L=2~(d=3)$, ChoiceGP outperforms the other three methods.
The last figure $L=2~(d=1,2,3,4,5)$ reports the accuracy for ChoiceGP in the medium sparsity regime as a function of the latent dimension $d$. It can be observed as the accuracy increases until $d=3$ (corresponding to the true latent dimension) and then remains stable.
This shows that there is no overfitting at the increasing of the latent dimension.

These experiments show that, when inconsistencies in preference assessments are due to multiple  conflicting utilities, it is better to allow Alice to express incomparability judgements instead of compelling Alice to always choose a preferred object.

\subsection{Real datasets}  
\label{sec:real} 
We now focus on five benchmark
datasets -- AM, EDM, Jura, Slump, Vehicle --  for multi-output regression problems. We  use the three output variables as utility functions to generate choice data.
For instance, in the Additive Manufacturing (AM) dataset, we consider 6 features (layer height, nozzle temperature, bed temperature, print speed	material,fan speed) and we use the three outputs (roughness, tension strength, elongation) to generate choice data. More details on the datasets is provided in Appendix \ref{app:real}. For each dataset, we set $|A_k|=2$ and use the  three output variables to generate a dense choice set $\{(C(A_k),A_k):~k=1,\dots,m\}$.  By using a sparsity level equal to $0.4$ we randomly generate a training set and use the remaining choice pairs as test set (we repeat this process 5 times  generating a total of $25$ datasets).
 
Since PGP, GPGP and PairGP cannot deal with conflicts among preferences (and so with choice data), we consider two common ways\footnote{Note that, there are criteria to deal with conflicts which  generate consistent preferences, for instance by weighting the utilities. In these cases, PGP and ChoiceGP are  best models.} to deal with these conflicts: (1) random selection; (2)  looking for the alternatives that are favoured by most (but not necessarily all) of the preference criteria. We generate preference data from the above choice datasets as follows. For each $A_k=\{{\bf x}_i,{\bf x}_j\}$ if $C(A_k)=\{{\bf x}_i\}$ then ${\bf x}_i \succ {\bf x}_j$. Instead, whenever $C(A_k)=\{{\bf x}_i,{\bf x}_j\}$, we generate  preference data in two ways: \textbf{random:}  coin flip: ${\bf x}_i \succ {\bf x}_j$ if Heads; ${\bf x}_j \succ {\bf x}_i$ if Tails; \textbf{majority rule:}   ${\bf x}_i \succ {\bf x}_j$ if ${\bf x}_i$ is better than  ${\bf x}_j$ with respect to 2 out of 3 outputs. The following table reports the averaged accuracy. 

 \begin{center}
   \scalebox{0.74}{
\begin{tabular}{c|ccccc}
&    \textbf{ChoiceGP} &  &   \textbf{PGP} &   \textbf{GPGP} & \textbf{PairGP} \\
\hline  
\multirow{2}{*}{AM} & \multirow{2}{*}{{0.90}} & maj.  & $0.84$ & $0.86$ & $0.87 $ \\
& & rand.  & $0.74 $ & $0.73 $ & $0.73 8$ \\
\cline{2-6}
\hline  
\multirow{2}{*}{EDM} & \multirow{2}{*}{{0.88}} & maj.  & $0.83 $ & $0.80 $ & $0.83$ \\
& & rand.  & $0.84 $ & $0.82 $ & $0.82 $ \\
\cline{2-6}
\hline  
\multirow{2}{*}{Jura} & \multirow{2}{*}{{0.91}} & maj.  & $0.87 $ & $0.87 $ & $0.87 $ \\
& & rand.  & $0.84$ & $0.82$ & $0.82$ \\
\cline{2-6}
\hline  
\multirow{2}{*}{Slump} & \multirow{2}{*}{{0.91}} & maj.  & {0.93} & $0.90 $ & $0.90 $ \\
& & rand.  & $0.83 $ & $0.79 $ & $0.79$ \\
\cline{2-6}
\hline  
\multirow{2}{*}{Vehicle} & \multirow{2}{*}{{0.93}} & maj.  & $0.89$ & $0.90$ & $0.90$ \\
& & rand.  & $0.80$ & $0.80$ & $0.80$ \\
\cline{2-6}
\hline 
\end{tabular}}
 \end{center}
It can be noted that ChoiceGP overall outperforms PGP, GPGP and PairGP in both the \textit{random} and \textit{majority} rule scenario. The only exception is the dataset \textit{slump}, where PGP has higher accuracy in  the \textit{majority} rule scenario. In appendix \ref{app:real}, we perform a statistical analysis of the results to show that the difference between the algorithms is practically and statistically significant.

 Finally, we run the latent-dimension selection procedure on the five datasets. The below table reports the PSIS-LOO for different values of the dimension $d$ in one of the 5 MC repetitions. It can be observed how the selection procedure always selects the true dimension $d=3$.
 This happens consistently in 5 out of the 5 repetitions and demonstrates both the accuracy and reliability of the proposed latent dimension selection procedure.

  \begin{center}
   \scalebox{0.8}{
 \begin{tabular}{c|ccccc}   
    {$d$} & AM& EDM &  Jura & Slump & Vehicle  \\ \midrule
    1  & -16335 & -5646 & -11182  & -10754  & -12692\\
    2  & -186  & -138 &  -254  &  -211  &-272\\
    3  & \textbf{-152} & \textbf{-136} & \textbf{-194}  & \textbf{-179} &\textbf{-174}\\
    4  &  -160&  -158 &  -207 &  -199 & -179 \\  
    5  &   -170 & -173 &  -223 & -234  & -205 \\\bottomrule
\end{tabular}}
\end{center}      
  
\section{Conclusions}
 We have developed a Gaussian Process based-method to learn choice functions from choice data via Pareto rationalization. As future work, we plan extend this model to deal with context-dependent choice functions \citep{pfannschmidt2022learning} as well as consider choice under uncertainty \citep{seidenfeld2010coherent,de2019interpreting}.

 \begin{acknowledgements}
 For the first author, this publication has emanated from research conducted with the financial support of the EU Commission Recovery and Resilience Facility under the Science Foundation Ireland Future Digital Challenge Grant Number 22/NCF/FD/10827.
 The authors thank Karlson Pfannschmidt for sharing the code for ChoiceNN and Siu Lun Chau for sharing the data generation code for the numerical experiments in the GPGP paper. ~\\
 
\paragraph{Software}
 The Python implementation of ChoiceGP is available at https://github.com/benavoli/ChoiceGP
\end{acknowledgements}

\bibliographystyle{apalike}
\bibliography{biblio}

\appendix
\onecolumn
\newpage

 \section{Vectorisation of the likelihood}
 \label{app:like}
 The product in the second row in \eqref{eq:likelihoodexpanse0} is a probabilistic relaxation of \eqref{eq:likcondpareto2}. Since it always involves comparisons between pairs of objects, it can be vectorized as follows:
   \begin{equation}
  \label{eq:likelihoodexpanseapp0}
 \begin{aligned}
      &\prod_{k=1}^m\prod_{\{{\bf o},{\bf v}\} \in C_2(A_k)}\left( 1-\prod_{i=1}^d \Phi\left(\frac{u_i({\bf o})-u_i({\bf v})}{\sigma}\right)-\prod_{i=1}^d \Phi\left(\frac{u_i({\bf v})-u_i({\bf o})}{\sigma}\right)\right)
   = \prod_{{\bf a}_i \in \mathcal{A}}\left( 1- \Phi_d\left(\frac{{\bf a}_i{\bf u}(X)}{\sigma}\right)-\Phi_d\left(\frac{-{\bf a}_i{\bf u}(X)}{\sigma}\right)\right),\\
   \end{aligned}
 \end{equation}
 where there is a vector ${\bf a}_k \in \mathbb{R}^{1 \times t}$ for each pairs $\{{\bf x}_i,{\bf x}_j\} \in C(A_k)$  with ${\bf x}_i\neq {\bf x}_j$. ${\bf a}_k$ is a zero vector whose $i$-th and $j$-th elements are equal to $1$ and respectively, $-1$, and $\Phi_d$ is the CDF of d-dimensional standard multivariate Gaussian distribution. 
      
 The product in the last row in \eqref{eq:likelihoodexpanse0} is a probabilistic relaxation of \eqref{eq:likcondpareto1}. It cannot be easily vectorized because $\prod_{{\bf o} \in C(A_k)}$ has a varying number of terms depending on $k$.   To overcome this issue we assume, as usually done in decision theory (see for instance \citep[Sec.3.4.2]{parmigiani2009decision}), the existence a worst object $\boldsymbol{\omega}\in \mathcal{X}$, that is an object such that $u_i(\boldsymbol{\omega})=-\infty$ for each $i=1,\dots,d$.
 This allows us to compare any ${\bf v} \in C(A_k)$ with the same number of elements (either ${\bf o} \in C(A_k)$ or $\boldsymbol{\omega}$).
 
 Assume for instance that $|A_k|=5$, $R(A_k)=\{{\bf v}_1,{\bf v}_2,{\bf v}_3\}$ and $C(A_k)=\{{\bf o}_1,{\bf o}_2\}$, then the  product in the last row in \eqref{eq:likelihoodexpanse0}
    \begin{equation}
  \label{eq:likelihoodexpanseapp1}
 \begin{aligned}
      &\prod_{{\bf o} \in C(A_k)} \left(1- \prod_{i=1}^d \Phi\left(\frac{u_i({\bf o})-u_i({\bf v}_1)}{\sigma}\right)\right)\prod_{{\bf o} \in C(A_k)} \left(1- \prod_{i=1}^d \Phi\left(\frac{u_i({\bf o})-u_i({\bf v}_2)}{\sigma}\right)\right)\prod_{{\bf o} \in C(A_k)} \left(1- \prod_{i=1}^d \Phi\left(\frac{u_i({\bf o})-u_i({\bf v}_1)}{\sigma}\right)\right)\\
   \end{aligned}
 \end{equation}
For each ${\bf v}_j$, we can write each product as
    \begin{equation}
  \label{eq:likelihoodexpanseapp2}
 \begin{aligned}
      &\prod_{{\bf o} \in C(A_k)} \left(1- \prod_{i=1}^d \Phi\left(\frac{u_i({\bf o})-u_i({\bf v}_j)}{\sigma}\right)\right)
      = \left(1- \prod_{i=1}^d \Phi\left(\frac{u_i({\bf o}_1)-u_i({\bf v}_j)}{\sigma}\right)\right) \left(1- \prod_{i=1}^d \Phi\left(\frac{u_i({\bf o}_2)-u_i({\bf v}_j)}{\sigma}\right)\right)\\
       =& \left(1- \prod_{i=1}^d \Phi\left(\frac{u_i({\bf o}_1)-u_i({\bf v}_j)}{\sigma}\right)\right) \left(1- \prod_{i=1}^d \Phi\left(\frac{u_i({\bf o}_2)-u_i({\bf v}_j)}{\sigma}\right)\right)\left(1- \prod_{i=1}^d \Phi\left(\frac{u_i(\boldsymbol{\omega})-u_i({\bf v}_j)}{\sigma}\right)\right)^3\\
   =& \prod_{{\bf b}_i \in \mathcal{B}}\left( 1- \Phi_d\left(\frac{{\bf b}_i{\bf u}(\tilde{X})}{\sigma}\right)\right),\\
   \end{aligned}
 \end{equation}
 where $\tilde{X}=[X,\boldsymbol{\omega}]^\top$. ${\bf b}_i \in \mathbb{R}^{1 \times (t+1)}$
 for each compared pairs $\{{\bf x}_i,{\bf x}_j\}$  with ${\bf x}_i\neq {\bf x}_j$, is a zero vector whose $i$-th and $j$-th elements are equal to $1$ and respectively, $-1$,
 
 \section{Label switching problem}
 \label{app:LA} 
The Laplace Approximation (LA) cannot be applied due to the so-called `label switching' problem, which is caused by symmetry in the likelihood: any permutation of the labels $i=1,\dots,d$ yields the same likelihood. For this reason, the Hessian of the log-likelihood w.r.t.\ ${\bf u}(X)$ is in general an indefinite matrix and LA is not well-defined.

Consider for instance the case where $d=2$ 
and $C(A_1)=\{\bx_a,\bx_b\}$ is the only choice data we have, the likelihood is 
$$
L = \left(1- \prod_{i=1}^2 \Phi\left(\frac{u_i({\bf x}_a)-u_i({\bf x}_b)}{\sigma}\right)- \prod_{i=1}^2 \Phi\left(\frac{u_i({\bf x}_b)-u_i({\bf x}_a)}{\sigma}\right)\right)
$$
and it is symmetric to the switching of $u_1$ and $u_2$.

The Variational Approximation is also affected by this problem, but it is well-defined. It will simply converge to one of the symmetric (to label switching) components of the distribution.

  \section{Variational Inference}
 \label{app:VI} 
 We implemented our model using automatic-differentiation in Jax   \citep{jax2018github}. 
 
 For the Variational Approximation (VA), we use the implementation in \citep{opper2009variational}, which has $2t$ parameters with $t=|X|$. In particular, we consider the covariance matrix in \citep[Equation (10)]{opper2009variational}. This means we only need $t$ parameters for the covariance matrix of the VA distribution. This is an approximation, but it allows us reduce the computational load of ChoiceGP, which is composed by $d$ GPs.
 
Indeed, by exploiting the above parametrisation and the factorised prior \eqref{eq:prior}, our ChoiceGP model can be implemented efficiently. We need
 storing and inverting $d$  kernel matrices with dimension $t \times t$.

We initialise the Variational Approximation  with MAP estimate and then perform 5000 iterations.

 \section{Interpretation of the probit likelihood}
  \label{app:batch}
 There are two ways to interpret the likelihood \eqref{eq:likelcdf0}:
 \begin{enumerate}
  \item \textit{Limit of discernibility:} Alice may make mistakes when comparing two objects ${\bf x}_i,{\bf x}_j$ whose difference in utility is small (e.g., errors are inversely proportional to the difference between the two utilities $|u({\bf x}_i)-u({\bf x}_j)|$). 
  \item \textit{Noise:} the observed utility function differs from
the true utility function due to disturbances
(e.g., $o({\bf x}_i)=u({\bf x}_i)+\text{noise}$).
 \end{enumerate}
 In this second case, it is well known that
\begin{equation}
    \label{eq:likelcdf0aa}
p({\bf x}_i \succ {\bf x}_j|u)=\Phi\left(\frac{u({\bf x}_i)-u({\bf x}_j)}{\sqrt{2}\sigma}\right)=\int I_{u({\bf x}_i)+w_i-u({\bf x}_j)-w_j>0}N(w_i;0,\sigma^2)N(w_j;0,\sigma^2)dw_i dw_j.
\end{equation}
 
 There is no  correct interpretation -- it depends on the  ``error-model'' we assume to account for  the inconsistencies in the subject's preferences. 
 
 For instance, for the computer example in Section \ref{sec:intro}, assuming that inconsistencies  are due to a Gaussian noise model does not make much sense. The features of the computer are observed exactly (without any noise). Instead, it is reasonable to assume that two different computers,  which only have slightly different characteristics, are indiscernible for Alice. For this reason, she may state inconsistent preferences when comparing them.
 
 Similarly, there may be cases where the utility is observed through a noisy measurement and, therefore, the second interpretation is more correct in this case.
 
 The issue arises when we compare the same objects multiple times. Assuming that Alice chooses ${\bf o}$ and discards the elements in $R(A_k)$, this leads to the following batch-likelihood  
 \begin{equation}
  \label{eq:like1}
   \begin{aligned}
      \int &\left(\prod_{{\bf v} \in R(A_k)}\Phi\left(\frac{u({\bf o})+w_k-u({\bf v})}{\sigma}\right)\right)N(w_k;0,\sigma^2)dw_k,\\
   \end{aligned}
 \end{equation}
 for the case $d=1$ (single utility) and noise-model, which is different from the limit-of-discernibility error-model
 \begin{equation}
  \label{eq:like2}
 \begin{aligned}
       \prod_{{\bf v} \in R(A_k)}\Phi\left(\frac{u({\bf o})-u({\bf v})}{\sigma}\right).\\
   \end{aligned}
 \end{equation}
 As stated in Proposition \ref{prop:comparison} (see proof below), \eqref{eq:like2} is a lower bound for  \eqref{eq:like1}. This means that either
 \begin{itemize}
  \item assuming \eqref{eq:like1} when \eqref{eq:like2} is the true error-model, or
  \item assuming \eqref{eq:like2} when \eqref{eq:like1} is the true error-model
 \end{itemize}
 may lead to a biased posterior. We will further investigate the difference between these two models in future work.

 .
 \begin{proposition}
 \label{prop:comparison}
  The  likelihood \eqref{eq:likelihood1} is a lower bound of the the batch-preference likelihood:
    \begin{equation}
  \label{eq:likelihoodbatch00}
 \begin{aligned}
       \int &\left(\prod_{{\bf v} \in R(A_k)}\Phi\left(\frac{u({\bf o})+w_k-u({\bf v})}{\sigma}\right)\right)N(w_k;0,\sigma^2)dw_k.\\
   \end{aligned}
 \end{equation}
 \end{proposition}

\begin{proof}
We are going to use the following results. \\
  
\textbf{Result:}  If ${\bf v}=[v_1,\dots,v_d]$ are independent, then for any increasing functions $h$ and $g$ of $n$ variables:
 $$
 E[f({\bf v})g({\bf v})]\geq E[f({\bf v})]E[g({\bf v})].
 $$
 The proof can be found in \cite[Sec.9.9]{ROSS2013153}     
   
  Consider the  likelihood \eqref{eq:likelihoodbatch00} 
$$
 \begin{aligned}
     \int &\left(\prod_{{\bf v} \in R(A_k)}\Phi\left(\frac{u({\bf o})+w_k-u({\bf v})}{\sigma}\right)\right)N(w_k;0,\sigma^2)dw_k.\\
   \end{aligned}
$$    
and note that all the functions between $(\dots)$ are monotone increasing in $w_k$. Therefore, we can exploit the above result to derive that
$$
 \begin{aligned}
      & \int \left(\prod_{{\bf v} \in R(A_k)}\Phi\left(\frac{u({\bf o})+w_k-u({\bf v})}{\sigma}\right)\right)N(w_k;0,\sigma^2)dw_k\geq  \prod_{{\bf v} \in R(A_k)}\Phi\left(\frac{u({\bf o})-u({\bf v})}{\sigma}\right). 
   \end{aligned}
$$   
 
\end{proof}

\section{ChoiceNN vs.\ ChoiceGP}
\label{sec:counterexample}

We illustrate the issue with ChoiceNN considering the 1D utility function $u(x)=\cos(5x)+\exp\left(-\tfrac{x^2}{8}\right)$ with $x \in [-2.6,2.6]$
in Figure \ref{fig:trueuNN}.
    
\begin{figure}[h]
\centering
 \begin{tabular}{c}
\includegraphics[height=3cm]{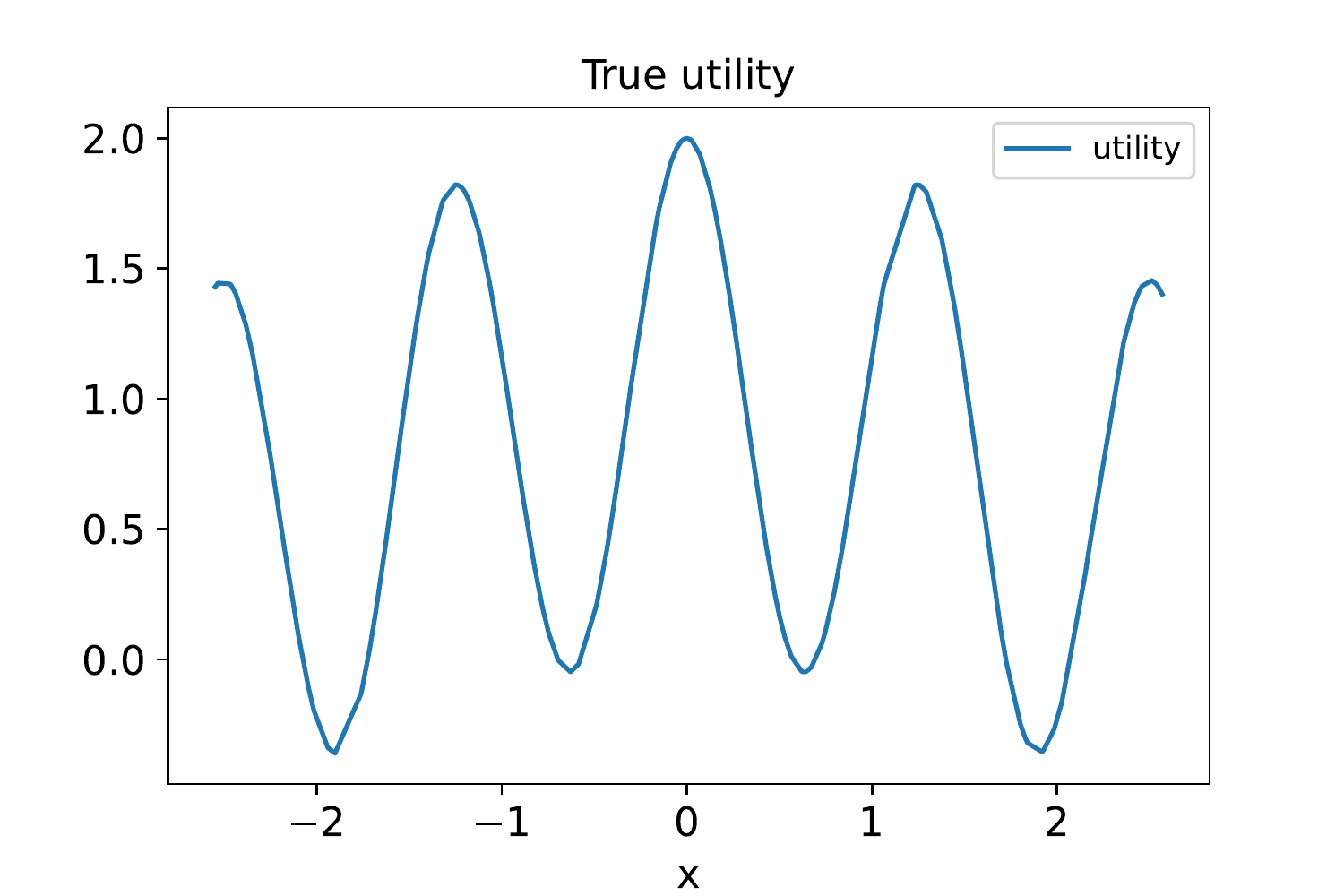}
                                                                                                                                       \end{tabular}
	\caption{True utility function}
	\label{fig:trueuNN}
\end{figure}

We used $u$ to generate choice data.  We sampled $150$ inputs $x_i$ at random in $ [-2.6,2.6]$. We then generated  $m=500$ random subsets $\{A_k\}_{k=1}^m$ of the 500 points each one of size $|A_k|=2$  and computed the corresponding choice pairs $(C(A_k),A_k)$ based on  $u$.

We ran ChoiceNN with fixed  latent dimension $d=1$
The learned utility function is shown in Figure \ref{fig:appdiff} (a), which is reasonably consistent with the true utility.

We then ran ChoiceNN with fixed  latent dimension $d=2$ (the true latent dimension is one) and reported the estimated utility functions, for two different random initialisation of the parameters of the NN, in Figure \ref{fig:appdiff} (b) and, respectively (c). It can be noticed that the model  converged to two different local optima. In both cases, the learned utility functions are not Pareto-consistent with the choice-data. In other words, the model is not able to find a utility representation of the choice data and, therefore, it is not able to make correct predictions.
We have tried different NN architectures (number of layers and number of nodes) as well as different values of the hyper-parameters for ChoiceNN, but the issue remains.

 The disadvantage of a nonlinear parametric method, like ChoiceNN, is the fact that the latent utility functions depend nonlinearly on the parameters. Instead, in ChoiceGP, the utility functions (at the training data) are part of the  the variational parameters and, therefore, can be more easily optimised to  satisfy the Pareto-consistency implied by the choice data.

 Figure \ref{fig:trueChoiceGP} shows the utilities learned by ChoiceGP with $d=2$. They coincide.
 This shows that ChoiceGP is able to easily understand that the true latent dimension is one.
 Moreover, the learned utility basically coincides with the true utility in Figure \ref{fig:trueuNN}
 (apart from a scaling factor, which cannot be estimated from the data).

\begin{figure*}[h]
	\centering
	\begin{tabular}{ccc}
		\includegraphics[height=3cm,trim={0.0cm 0.0cm 1.5cm 0.0cm }, clip]{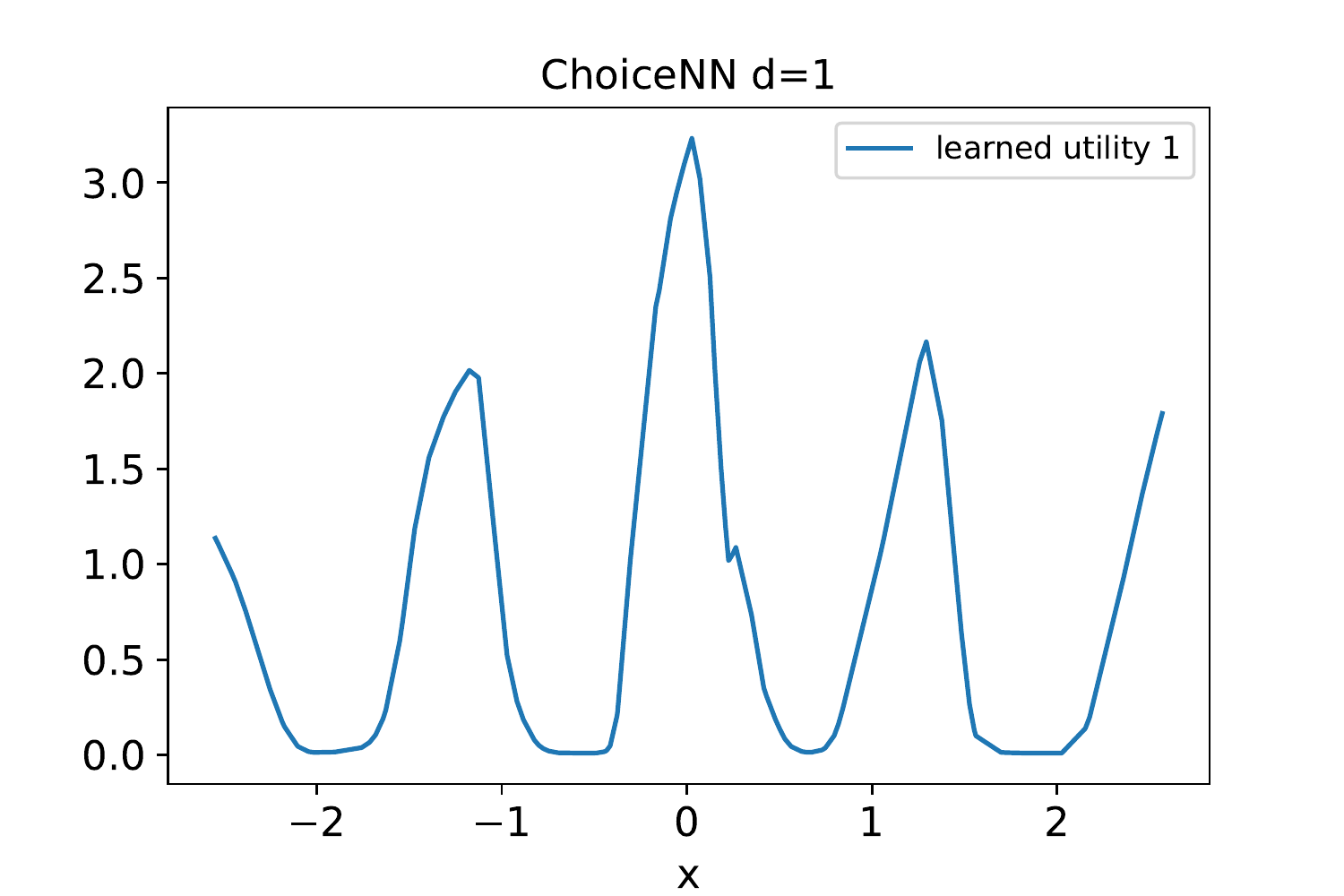} &
		\includegraphics[height=3.0cm,trim={0.0cm 0.0cm 1.5cm 0.0cm }, clip]{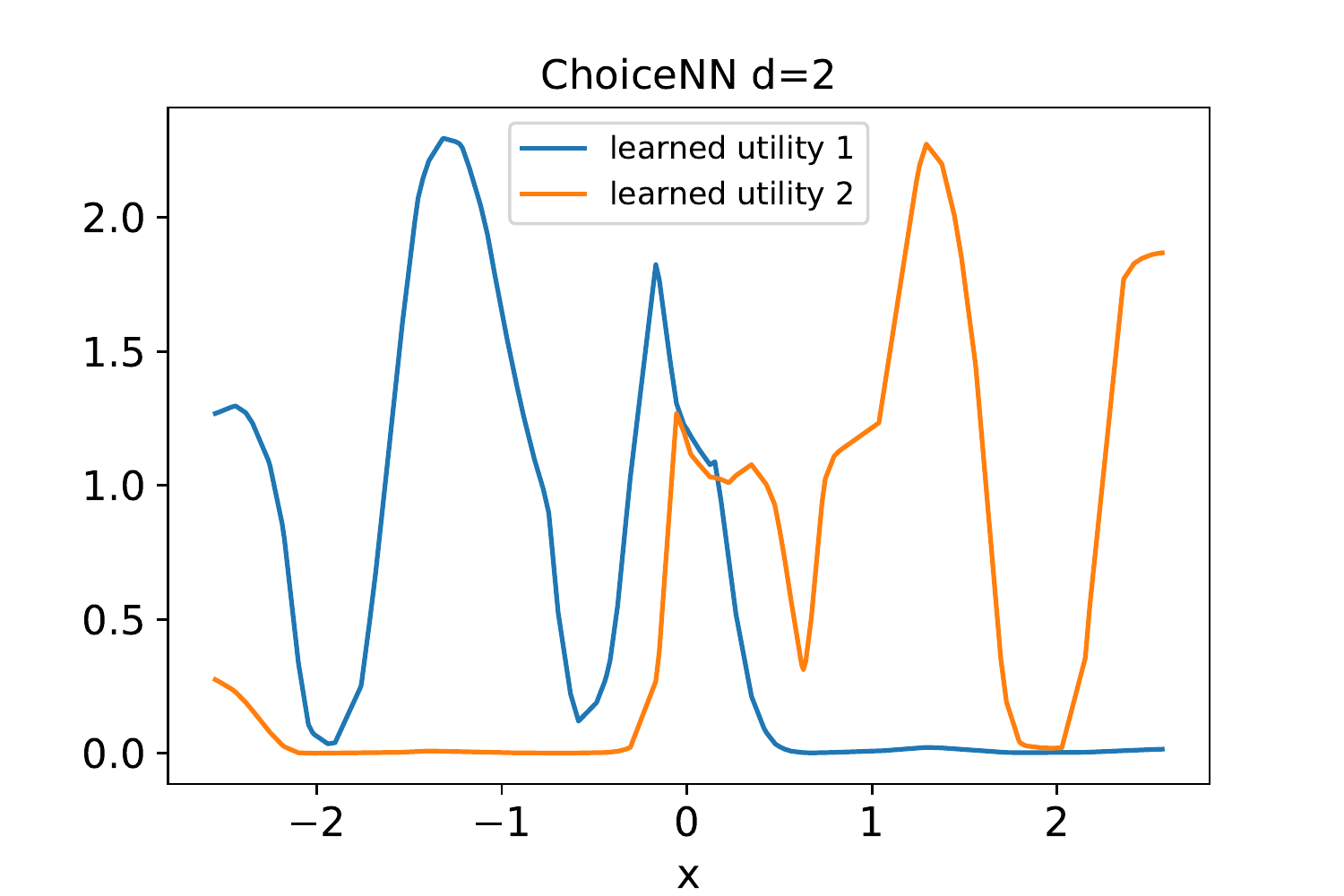} &
		\includegraphics[height=3.0cm,trim={0.0cm 0.0cm 1.5cm 0.0cm }, clip]{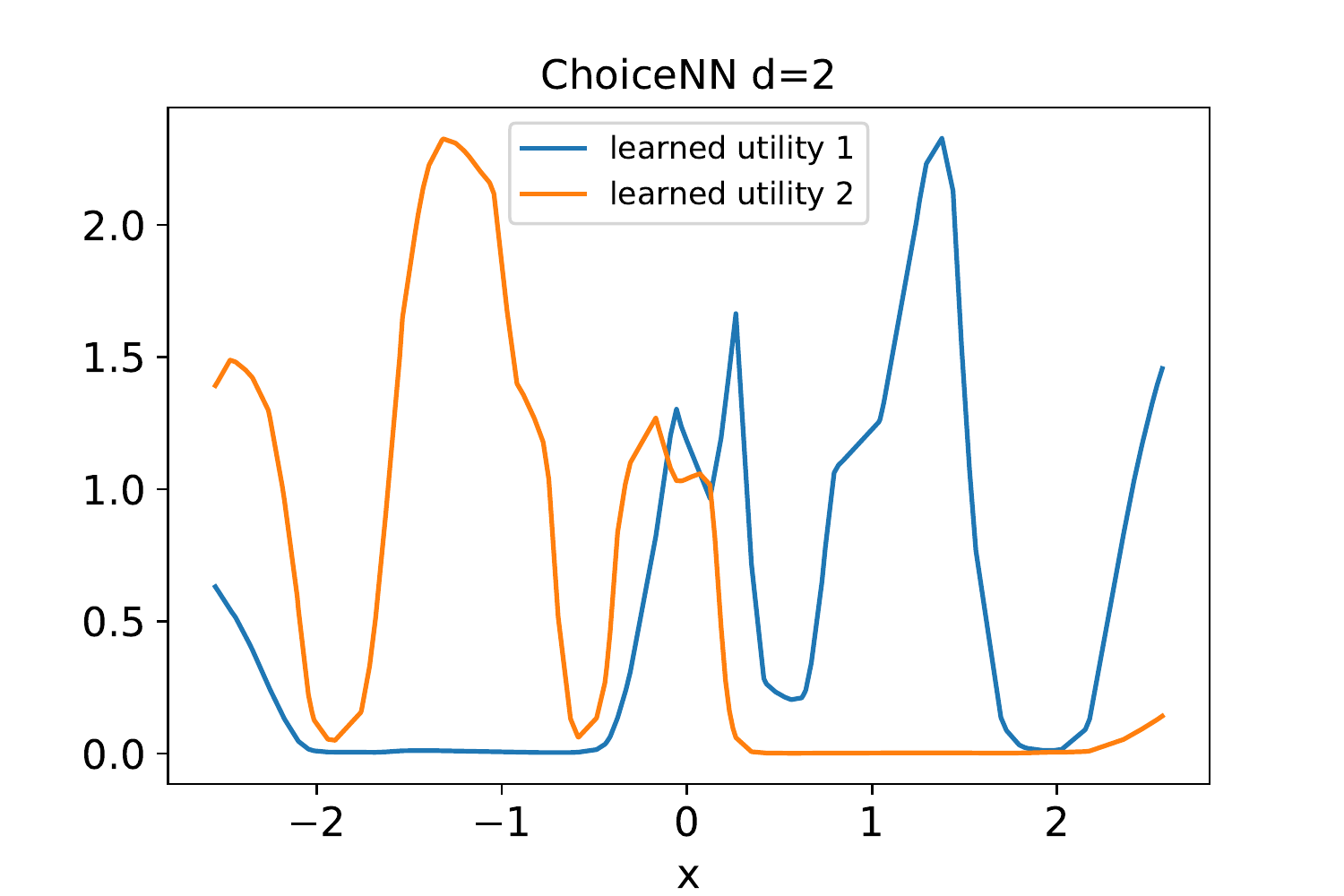} \\
		(a) & (b) & (c)\\
	\end{tabular}  
	\caption{Learned utilities via ChoiceNN}
	\label{fig:appdiff}
\end{figure*}

\begin{figure}[h]
\centering
 \begin{tabular}{c}
\includegraphics[height=3cm]{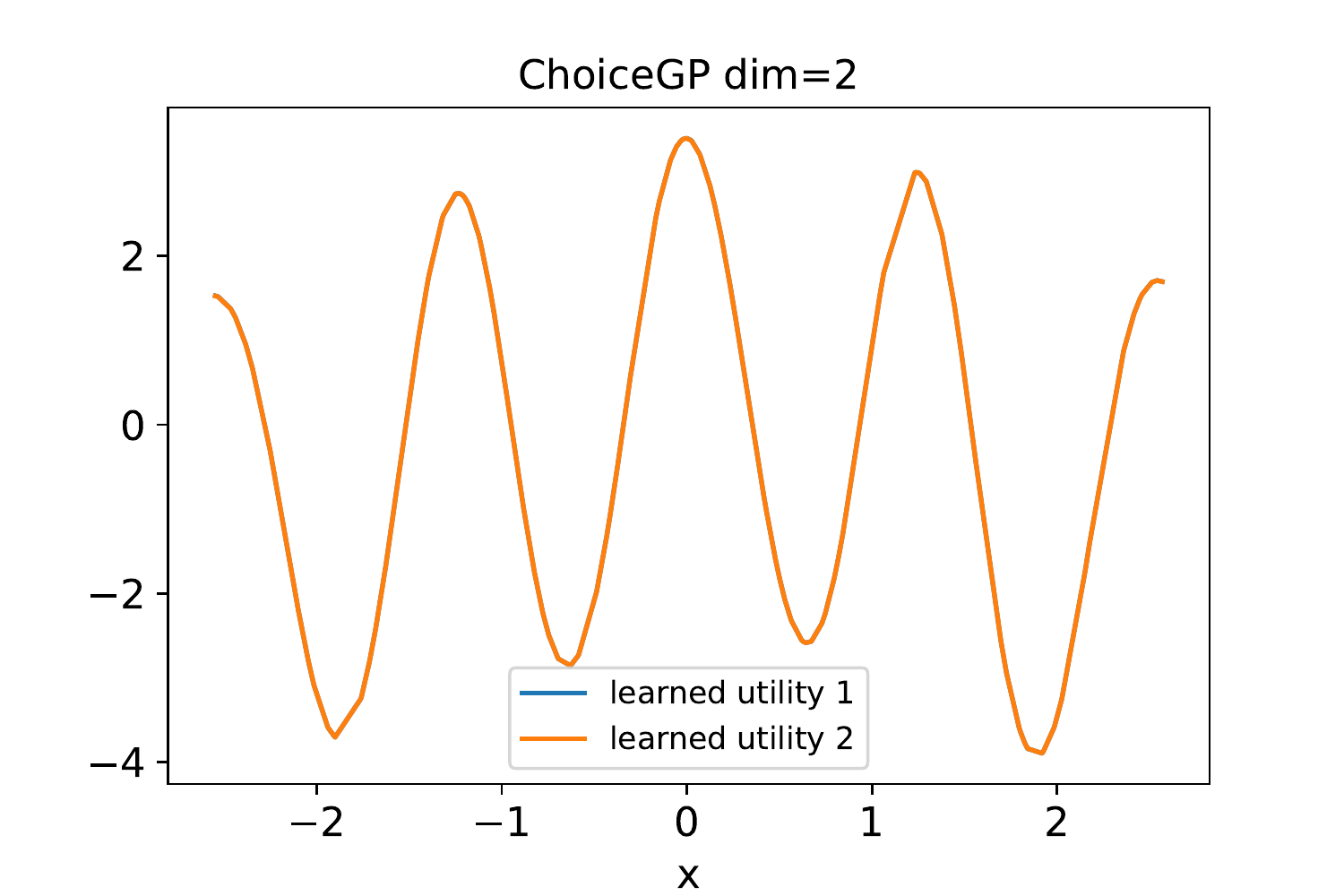}
                                                                                                                                       \end{tabular}
	\caption{Learned utilities via ChoiceGP}
	\label{fig:trueChoiceGP}
\end{figure}

\clearpage 

 \section{Real-datasets}
 \label{app:real}
  Table \ref{tab:charac} displays the characteristics of the considered datasets. 
  
\begin{table}[H]
		\begin{center}
			{\small
			   \scalebox{0.8}{
				\begin{tabular}{lccc}
					\hline
					{\bf Dataset}  & {\bf \#Features} & {\bf \#Outputs} \\
					\hline
					AM &  6 & 3 \\
					EDM &  4 & 3 \\
					jura &  6 & 3 \\
                    slump  & 7 & 3 \\
                    vehicle  & 5 & 3 \\
					\hline
				\end{tabular}}
			}
		\end{center}   
		\caption{Characteristics of the datasets.}
		\label{tab:charac}
	\end{table}
	
The first 4 datasets are standard datasets used in multi-target regression. The ``vehicle dataset'' has been obtained from the  Vehicle-Safety model\footnote{A model that determines the thickness of five reinforced components of a vehicle's frontal frame \citep{yang2005metamodeling}} using a latin-hypercube design of experiment.  We have included the datasets in our repository together with the code to replicate the experiments. 

We have implemented GPGP, PGP and PairGP in GPy \citep{gpy2014}. For ChoiceNN, we use the implementation provided by the authors  \cite{pfannschmidt2020learning}.  

As shown in the average-accuracy table in Section \ref{sec:real}, ChoiceGP has a higher average accuracy than the other models. This claim is also supported by a statistical analysis as we will show hereafter. 

We have compared ChoiceGP against PGP, GPGP and PairGP for the majority-rule using the pairwise Bayesian hierarchical hypothesis testing model \citep{corani2017statistical}. The test accounts for the correlation between the paired differences of accuracy due to  the overlapping training
sets built during cross-validation. This test  declares two models practically equivalent when the difference of accuracy is less than 0.01 (1\%). The interval $[-0.01, 0.01]$ thus
defines a region of practical equivalence (rope) for the performance of the models. For instance for the pair (ChoiceGP,PGP), the test returns the posterior samples of the  probability vector  $[p(ChoiceGP > PGP), p(ChoiceGP \approx PGP), p(ChoiceGP < PGP)]$ and,
therefore, this posterior can be visualised in the probability simplex (Figure \ref{fig:hierarch}). 
For all the pairwise comparisons,  it can seen that the vast majority of the samples are in the region at the right bottom of the
triangle. This confirms that ChoiceGP is practically significantly better than the other three methods.
 Note that, we have only statistically compared the methods in the majority-rule scenario, because the differences are even larger in the random scenario.

\begin{figure}[h]
\centering  
 \includegraphics[width=5cm]{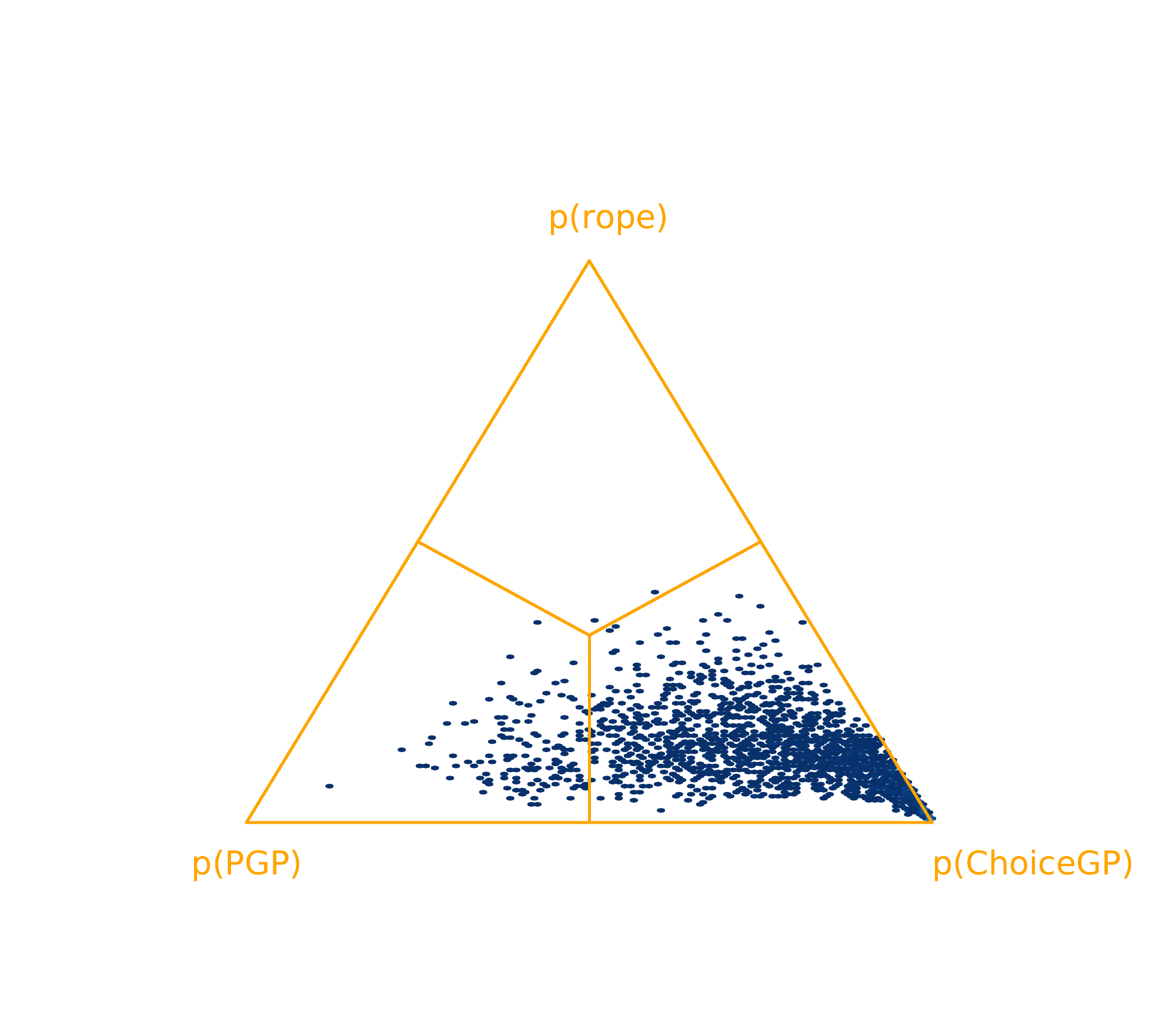}
  \includegraphics[width=5cm]{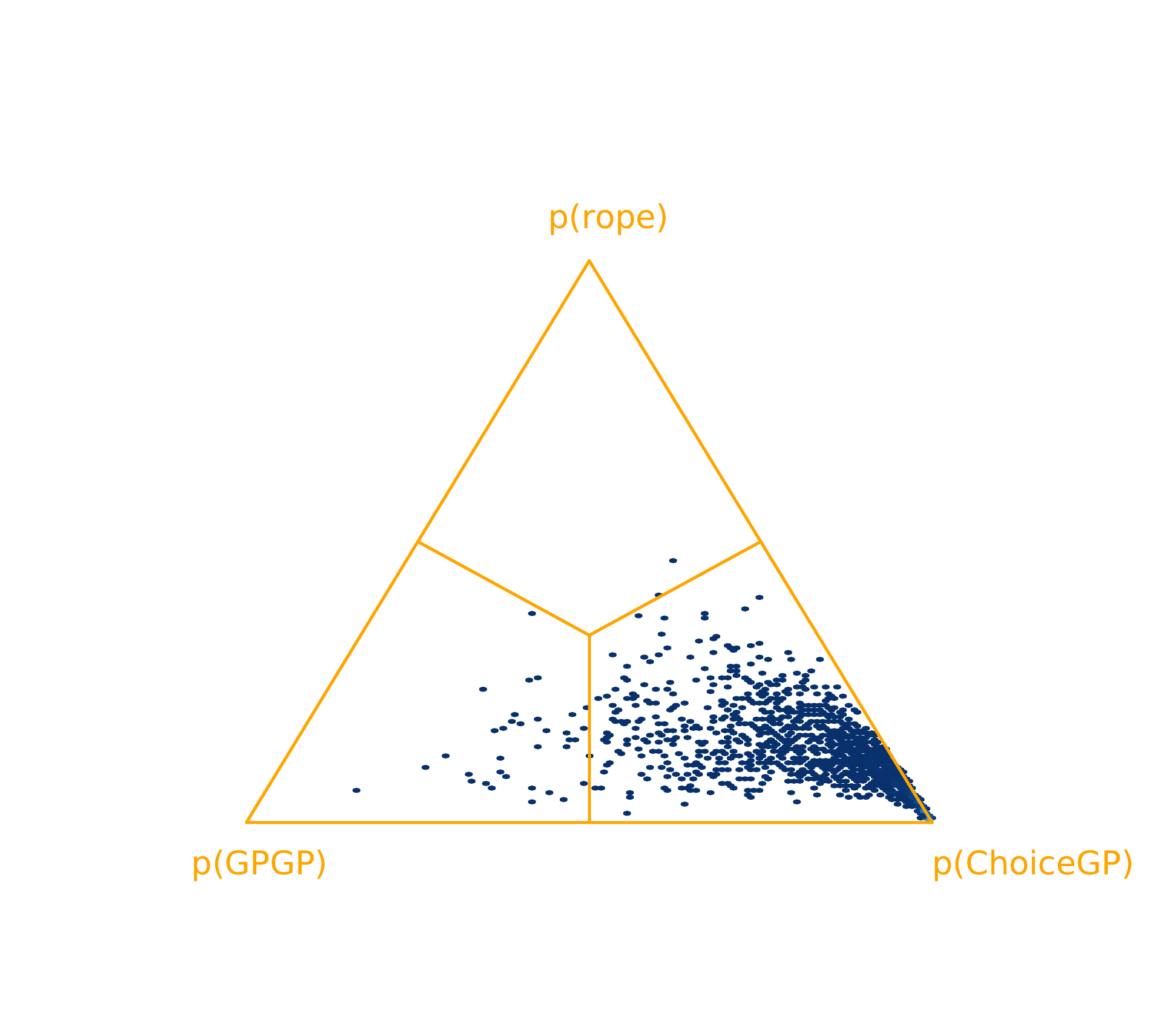}
   \includegraphics[width=5cm]{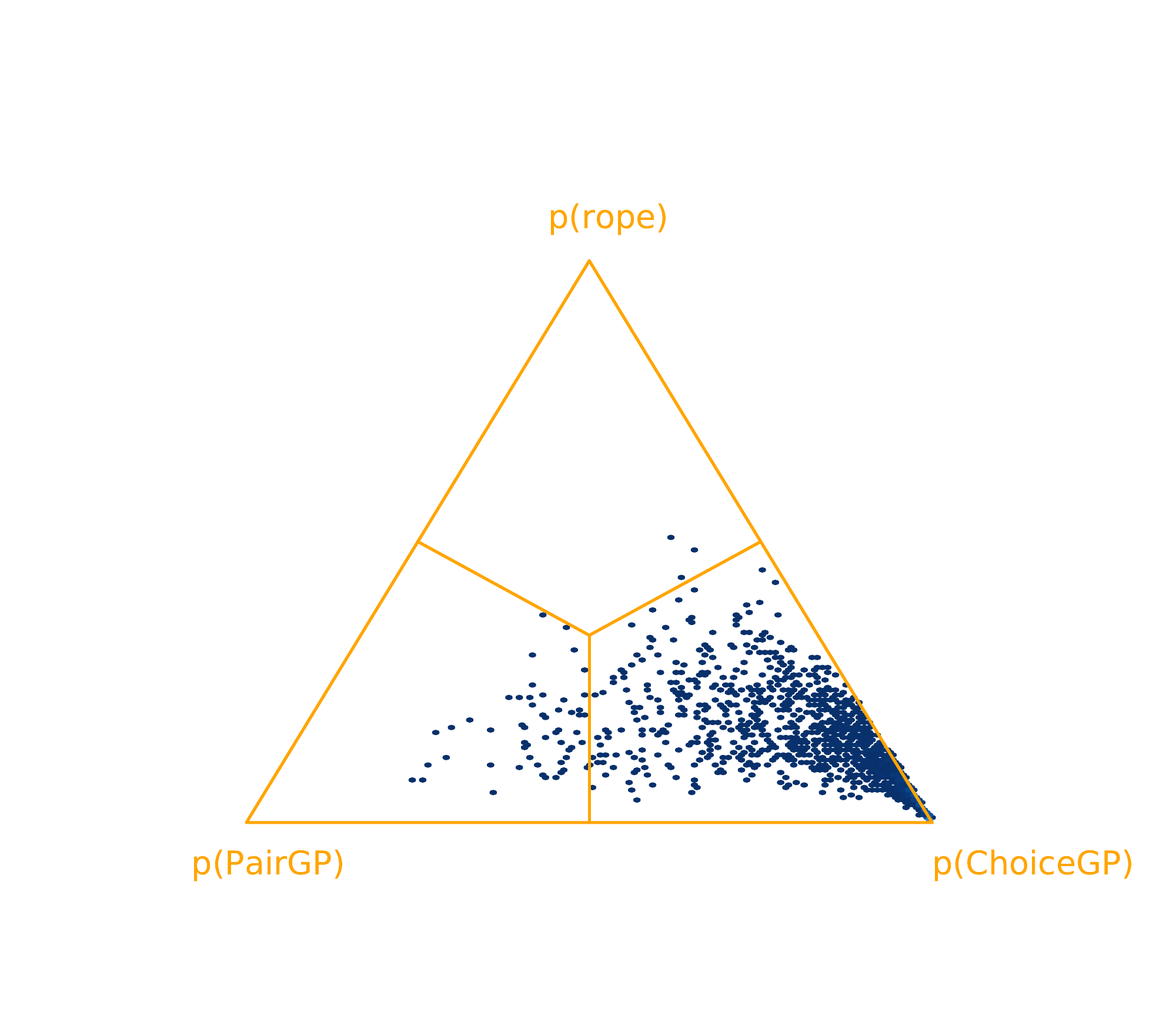}
   \caption{Posterior samples for the pairwise tests ChoiceGP vs. PGP, GPGP and, respectively, PairGP. This confirms that ChoiceGP is practically significantly better than the other three methods.}
   \label{fig:hierarch}
\end{figure}
\end{document}